%% file: main.tex
\documentclass[11pt]{article}

\usepackage[T1]{fontenc}
\usepackage{microtype}

\usepackage[margin=1.25in]{geometry}

\usepackage{amsmath,amsfonts,amssymb}
\usepackage{amsthm}
\usepackage{bm}
\usepackage{dsfont}
\usepackage{bbm}
\usepackage{latexsym}
\usepackage{rsfso}

\theoremstyle{plain}
\newtheorem{theorem}{Theorem}
\newtheorem{lemma}[theorem]{Lemma}
\newtheorem{proposition}[theorem]{Proposition}

\newtheorem{condition}[theorem]{Condition}

\theoremstyle{definition}
\newtheorem{definition}[theorem]{Definition}
\newtheorem{assumption}[theorem]{Assumption}

\theoremstyle{remark}
\newtheorem{remark}[theorem]{Remark}

\usepackage{graphicx}
\usepackage{subcaption}
\usepackage{booktabs}

\usepackage{xcolor}
\definecolor{darkblack}{rgb}{0, 0, 0.5}
\usepackage{hyperref}
\hypersetup{colorlinks=true, citecolor=darkblack, linkcolor=darkblack, urlcolor=darkblack}

\usepackage{natbib}
\bibpunct[, ]{(}{)}{,}{a}{}{,}

\usepackage{cleveref}
\crefname{equation}{Eq.}{Eqs.}
\Crefname{equation}{Eq.}{Eqs.}
\crefname{figure}{Figure}{Figures}
\crefname{table}{Table}{Tables}
\crefname{section}{Section}{Sections}
\crefname{theorem}{Thm.}{Thms.}
\Crefname{theorem}{Theorem}{Theorems}
\crefname{lemma}{Lem.}{Lems.}
\Crefname{lemma}{Lemma}{Lemmas}
\crefname{proposition}{Prop.}{Props.}
\Crefname{proposition}{Proposition}{Propositions}
\crefname{corollary}{Cor.}{Cors.}
\Crefname{corollary}{Corollary}{Corollaries}
\crefname{assumption}{Assumption}{Assumptions}
\Crefname{assumption}{Assumption}{Assumptions}
\crefname{definition}{Definition}{Definitions}
\crefname{remark}{Remark}{Remarks}
\crefname{condition}{Condition}{Conditions}
\crefname{example}{Example}{Examples}

\usepackage{algorithm}
\usepackage{algpseudocode}

\usepackage[T1]{fontenc}
\usepackage{url}
\usepackage{ulem}
\usepackage{cancel}
\usepackage{comment}
\usepackage{enumerate}
\usepackage{enumitem}
\usepackage{xspace}
\usepackage{xifthen}
\usepackage{makecell}
\usepackage{changepage}
\usepackage{lineno}
\usepackage{tikz}
\usetikzlibrary{positioning,calc,arrows.meta,decorations.pathreplacing}

\usepackage{authblk}

\newcommand{\E}{\mathbb{E}}

\newcommand{\R}{\mathbb{R}}
\newcommand{\N}{\mathbb{N}}

\newcommand{\Prob}{\mathbb{P}}
\providecommand{\abs}[1]{\lvert#1\rvert}

\newcommand{\coloneqq}{\mathrel{\mathop:}=}

\providecommand{\deleted}[1]{}

\newcommand{\tianyi}[1]{}

\title{\textbf{Throughput-Optimal Scheduling Algorithms for LLM Inference and AI Agents}}

\author[1]{J.G. Dai}
\author[2]{Tianze Deng}
\author[1]{Yueying Li}
\author[3]{Tianyi Peng}
\affil[1]{School of Operations Research and Information Engineering, Cornell University\\
  \texttt{\{jd694,yl3469\}@cornell.edu}}
\affil[2]{Operations Management, Booth School of Business, University of Chicago\\
  \texttt{tianze1@uchicago.edu}}
\affil[3]{Decision, Risk and Operations, Columbia Business School, Columbia University\\
  \texttt{tianyi.peng@columbia.edu}}
\date{}

\begin{document}

\maketitle

\begin{abstract}
As demand for Large Language Models (LLMs) and AI agents grows rapidly, optimizing systems for efficient LLM inference becomes critical. While significant efforts have targeted system-level engineering, little has been explored from a mathematical modeling and queueing perspective.
In this paper, we develop the queueing fundamentals for LLM inference. In particular, we study the throughput aspect of LLM inference systems. We prove that a large class of `work-conserving' scheduling algorithms achieve maximum throughput for both individual requests and AI-agent workloads with directed acyclic graph (DAG) and fork-join routing topologies, establishing `work-conserving' as a key design principle for practitioners. Technically, we develop a fluid-limit framework for multi-class batched processing networks under $K$-FCFS scheduling, which may be of independent interest. Evaluations of real-world systems confirm that Orca and Sarathi-Serve are throughput-optimal, reassuring practitioners, while FasterTransformer and vanilla vLLM are not maximally stable and should be used with caution. Our analysis also reveals how constraints such as batch size limits and cyclic routing topologies complicate the throughput picture, pointing to rich open questions at the intersection of queueing theory and LLM system design.
\end{abstract}

\medskip
\noindent\textbf{Keywords:} LLM inference, queueing theory, scheduling algorithms, throughput optimality, AI agents

\bigskip

\input{sections/sec1_introduction.tex}

\input{sections/sec2_llm_inference_setup.tex}

\input{sections/sec3_stochastic_modeling.tex}

\input{sections/sec4_work_conservingness.tex}

\input{sections/sec5_ai_agent_workloads.tex}

\input{sections/sec6_batch_size_constraint.tex}

\input{sections/sec7_future_directions.tex}

\bibliography{main}
\bibliographystyle{plainnat}

\appendix

\input{sections/appendix_a_batch_processing_time.tex}

\input{sections/appendix_b_fluid_model_proof.tex}

\input{sections/appendix_c_lyapunov_proof.tex}

\input{sections/appendix_e_fork_join_proof.tex}
\input{sections/appendix_d_convex_hull_proof.tex}

\end{document}

%% file: sections/sec1_introduction.tex
\section{Introduction}

Large Language Models (LLMs) have become the backbone of many AI-driven applications, requiring efficient LLM inference systems to meet latency and throughput requirements. LLM inference systems process each user request in two phases: \textit{prefill} (prompt processing) phase followed by \textit{decode} (token generation) phase, each with distinct computational and memory bandwidth characteristics. Since LLM servers generate tokens autoregressively\footnote{A token in an autoregressive transformer is the basic input or output unit—typically a word, subword, or character—that the model sequentially predicts based on previously generated tokens.}, processing a single request requires running the LLM for multiple iterations, with each iteration generating one output token.  To optimize GPU utilization, it is important to batch processing multiple requests simultaneously. This paper studies performance guarantees for  various scheduling algorithms in forming batches.

To maximize throughput while meeting latency requirements, several LLM inference systems with different scheduling algorithms have been developed; \cref{tab:scheduling-algorithms} summarizes the four most widely studied ones. FasterTransformer~\citep{fastertransformer} performs batching at the request level and prioritizes the decode phase; however, when there are not enough decode requests to fill a batch, GPU utilization suffers. Orca~\citep{yu2022orca} improves upon this by enabling token-level batching with mixed batching---combining prefill and decode tokens in the same batch---and prioritizing prefill tokens so that new requests enter the decoding stage sooner. Vanilla vLLM\footnote{We use ``vanilla vLLM'' to refer to the original vLLM design~\citep{kwon2023efficient} without chunked prefill. The latest vLLM with chunked prefill enabled is work-conserving~\citep{vllm2025chunked}.}~\citep{kwon2023efficient} also prioritizes prefill but does not support mixed batching: each batch contains either only prefill or only decode tokens. More recently, Sarathi-Serve~\citep{agrawal2023sarathi} introduces chunked prefill with decode prioritization, limiting both the number of tokens per batch and the length of prefill chunks, which prevents requests with long prefills from blocking decoding. In practice, Sarathi-Serve is gaining popularity and presents competitive empirical performance compared to other incumbents. \footnote{See \cref{sec:incumbent} for more details on these algorithms.} 

\begin{table}[ht]
\centering
\caption{Summary of scheduling algorithms for LLM inference. See \cref{sec:incumbent} for formal definitions.}
\label{tab:scheduling-algorithms}
\small
\begin{tabular}{llll}
\toprule
Algorithm & Priority & Batching & Work-conserving \\
\midrule
FasterTransformer~\citep{fastertransformer} & Decode-first & No mixed batching & No \\
Vanilla vLLM~\citep{kwon2023efficient} & Prefill-first & No mixed batching & No \\
Orca~\citep{yu2022orca} & Prefill-first & Mixed batching & Yes \\
Sarathi-Serve~\citep{agrawal2023sarathi} & Decode-first & Chunked prefill + mixed & Yes \\
\bottomrule
\end{tabular}
\end{table}

Despite significant advances from the systems community, the field lacks a cohesive theoretical framework for evaluating and comparing these scheduling algorithms. Practitioners are often left to decide which system to adopt based on workload-specific benchmarking~\citep{databricks_llm_2023} or profiling tools~\citep{pytorch_profiler,nvidia_nsight}. Such evaluation is costly, as it must be repeated across models, hardware configurations, and workload distributions. Queueing theory offers a principled alternative: the operations research community has developed robust theoretical foundations for scheduling and resource allocation across diverse domains, including cloud computing~\citep{mitzenmacher2001power,grosof_new_2023,tirmazi_borg_2020,patke_queue_2024,ghodsi_dominant_2011,zhang2023shepherd}, operating systems and networking~\citep{yu2021nsdi,stoica1998core,iyer2023achieving,li2024libpreemptible}, telecommunications~\citep{whitt1993tail,maglaris1988performance,canetti1995bounding}, and manufacturing and healthcare operations~\citep{hu_optimal_2022,chen_optimal_2025,zychlinski_managing_2023}.

Motivated by this gap, we analyze LLM inference systems through queueing theory to address the fundamental question \begin{quote} \it{What is the fundamental limit of the maximal throughput rate for an LLM inference system, and what classes of scheduling algorithms can achieve this limit?} \end{quote}

We develop a novel queueing  model of  the LLM server, identify a broad class of \textbf{work-conserving} scheduling algorithms
that are proved to be throughput optimal in our queueing model, which is based on realistic system profiling and performance modeling. Our analysis confirms that scheduling algorithms like those in Sarathi-Serve and Orca are throughput optimal, thereby offering practical assurance to system designers. In contrast, non-mixed batching algorithms---such as those in FasterTransformer and vanilla vLLM---are inherently suboptimal in throughput and can become unstable under moderate load. These findings contribute to a deeper theoretical foundation for LLM system design: to optimize latency performance, one should optimize among throughput-optimal scheduling algorithms.

Our work supports a vision stated recently in  \cite{mitzenmacher2025queueing}, 
\begin{quote}
    \textit{``LLM systems present a wealth of new queueing and scheduling challenges that expand on problems from traditional queueing systems. While modeling LLM systems remains largely unexplored
in the queueing community, we believe queueing theory insights may lead to better scheduling systems."}
\end{quote}

Specifically, our work extends batch queueing theory by addressing LLM-specific challenges: dual-phase processing with distinct resource profiles, dynamic batch formation, and request interdependencies. We further analyze \textbf{AI-agent workloads}, in which a network of LLM servers collaboratively processes agent-level tasks. We show that work-conserving scheduling remains throughput-optimal for single servers with multiple request classes, DAG routing networks, and fork-join topologies. However, inspired by the Rybko--Stolyar network~\citep{rybko1992ergodicity}, we construct an example showing that work-conserving scheduling can fail under cyclic routing---an outcome that may be unintuitive to practitioners. We also analyze the \textbf{maximal batch size constraint}, which limits the number of concurrent requests in a batch. This constraint fundamentally alters the stability region from a scalar threshold to a two-dimensional convex hull, and we demonstrate that even work-conserving algorithms can fail under certain configurations.

Technically, we develop a fluid-limit framework for multi-class batched processing networks under $K$-FCFS scheduling. This framework is the key analytical tool that enables our throughput-optimality results across single-server, DAG, and fork-join settings---results that, to the best of our knowledge, are new to the queueing literature for batch-service networks. The framework may be of independent interest to the queueing community.

Our contributions include:

\begin{itemize}
  \item A formal queueing-theoretic framework for LLM inference scheduling that captures the unique characteristics of the prefill and decode phases, explicitly modeling the \textit{batch processing time}. This modeling effort lowers the barrier for the queueing theory community to contribute to improving LLM inference systems.

  \item A rigorous theoretical analysis establishing that \textit{work-conserving} scheduling algorithms achieve maximum throughput in single-instance LLM inference, highlighting this as a core design principle for practitioners. 

  \item An examination of widely used scheduling algorithms revealing that Orca and Sarathi-Serve are throughput-optimal, while FasterTransformer and vanilla vLLM can become unstable under moderate load (\cref{sec:incumbent}). These findings support the adoption of Sarathi-Serve--type algorithms and offer practical guidance for system selection.

  \item An extension to AI-agent workloads modeled as multi-class batch-service processing networks (\cref{sec:AI-agent}). We prove throughput optimality of work-conserving scheduling under DAG and fork-join routing topologies, while showing that cyclic routing can lead to instability even under work-conserving policies.

  \item An analysis of the batch size constraint (\cref{sec:extension-to-active-request-constraint}), which reveals a richer two-dimensional stability region. We show that work-conserving algorithms can fail at certain operating points, highlighting the need for adaptive scheduling strategies.
\end{itemize}
Our work emphasizes the practical implications of scheduling algorithm design through both theoretical analysis and empirical validation on real-world workloads. Figure~\ref{fig:instable} illustrates how different scheduling algorithms handle request dynamics differently, with non-work-conserving algorithms leading to queue blowup in real systems. We hope this research contributes toward building a foundation for principled scheduling algorithm design in LLM inference systems and encourages cross-disciplinary collaboration between the queueing theory and systems communities.
    

\begin{figure}[h]
    \centering
    \begin{subfigure}[t]{0.49\linewidth}
        \centering
        \includegraphics[width=\linewidth]{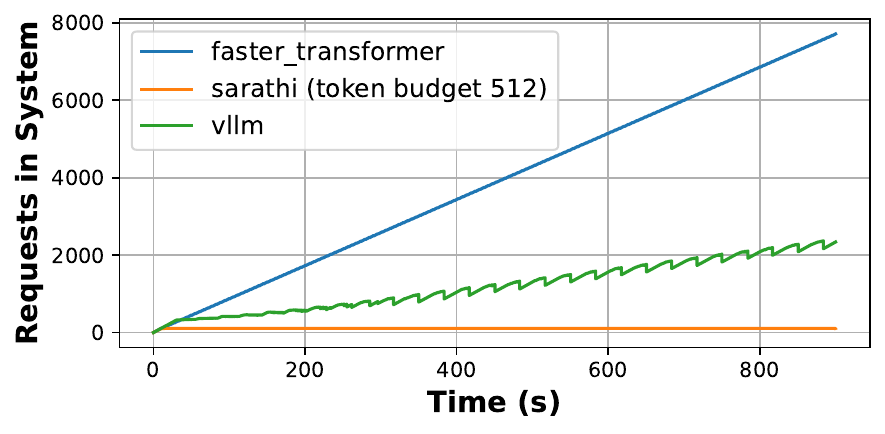}
        \label{fig:throughput}
    \end{subfigure}
    \hfill
    \begin{subfigure}[t]{0.49\linewidth}
        \centering
        \includegraphics[width=\linewidth]{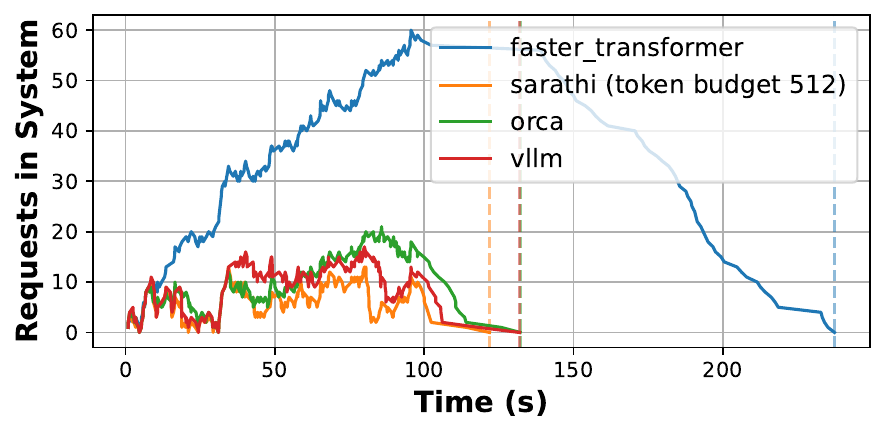}
    \end{subfigure}
    \caption{[Left] Experimental results demonstrating the instability of FasterTransformer and vLLM. The setup is using a CodeLlama-34B with A100, requests generator is Poisson with an average of 129 prefill and 112 decode tokens, arriving with 14 queries per second. [Right] Under a finite, fixed workload with the same arrival rate driven by production code traces~\citep{patel2023splitwise}, the throughput-optimal Sarathi-Serve algorithm is also better in latency (in terms of the total requests in the system across time (orange)). }
    \label{fig:instable}
\end{figure}

\subsection{Background}
\label{sec:tutorial}

\textbf{LLM generative inference.} When processing an input request, LLM inference consists of two phases: \textit{prefill} and \textit{decoding}. The prefill phase processes the entire request to compute the KV cache (Key-Value cache, which stores the attention keys and values for each transformer layer to avoid redundant computation) and generates the initial response token in a single step. The decoding phase then utilizes the prior contexts or KV cache to generate subsequent tokens one at a time. These phases have distinct resource utilization patterns: prefill is compute-bound, while decoding is memory I/O-bound. Conventional LLM inference systems place both phases on the same GPU group despite their different computational characteristics to maximize resource utilization. While we acknowledge that some recent works propose to disaggregate prefill and decode stages on separate GPUs~\citep{patel2023splitwise,zhong2024distserve}, the modeling of the p/d-disaggregation is considered analogous to our work.

\textbf{Inference goal.} Two critical metrics evaluate LLM inference performance: throughput and inference latency. Throughput measures the token generation rate within a given timeframe. Inference latency represents the time required to complete a request from start to finish; common latency metrics include TTFT (time to first token, measuring prefill latency) and TBT (time between tokens, measuring decoding latency).

For online inference---requests that need immediate responses (e.g., chatbot queries)---the common practice is to maximize throughput subject to latency service level objective (SLO) constraints, where SLO attainment measures the percentage (e.g., 99\%) of requests fulfilled within a predefined latency target. For offline inference---requests that can accommodate longer response times (e.g., large-scale document processing, overnight analytics)---the main focus is throughput.

In production cloud systems from a major cloud service provider, offline inference accounts for a non-trivial portion of the overall workload (\cref{fig:offline}), making throughput optimization a first-order concern. This paper focuses on throughput and leaves the study of SLO attainment to future work.

\begin{figure}
  \centering
  \includegraphics[width=0.75\linewidth]{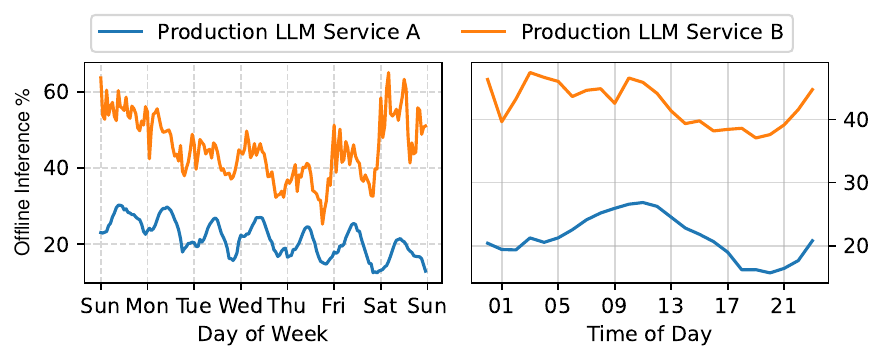}
  \caption{Percentage of offline inference demand for two LLM services (A and B) running in a production cloud data center over a week (left) and during a day (right), demonstrating the importance of optimizing the throughput of LLM scheduling~\citep{li2025ecoserve}. }
  \label{fig:offline}
\end{figure}



\textbf{Batch scheduling in LLM inference.} Modern LLM inference systems leverage continuous batching~\citep{yu2022orca}, which combines prefill and decode tokens in the same batch to improve GPU utilization. However, mixing the two phases creates interference: even a single long prefill can substantially degrade decode performance. Different scheduling algorithms manage this tension differently (\cref{tab:scheduling-algorithms}): FasterTransformer~\citep{fastertransformer} and vanilla vLLM~\citep{kwon2023efficient} avoid mixing entirely, Orca~\citep{yu2022orca} allows mixed batching with prefill priority, and Sarathi-Serve~\citep{agrawal2023sarathi} uses chunked prefill with decode priority. Understanding which of these strategies achieves optimal throughput is fundamentally a queueing-theoretic question.

Classical batch queueing theory has studied stochastic service capacity~\citep{chang2005performance,janssen2005analytic} and multi-class batch service~\citep{baetens2018delay,reddy1993scheduling}, but LLM inference introduces new challenges: dual-phase processing with distinct resource profiles, dynamic batch formation, and request interdependencies. This direction is attracting growing interest: \cite{mitzenmacher2025queueing} investigates scheduling with ML-guided prediction of token lengths. \cite{ao2025optimizing} study fluid-guided online scheduling under KV cache memory constraints, proposing threshold-based batching policies with near-optimal throughput guarantees; their work assumes knowledge of future decoding token lengths. Other concurrent works study offline scheduling with heterogeneous request lengths~\citep{wang2025llmserving}, robust scheduling under output-length prediction uncertainty~\citep{chen2025adaptively}, and SLO-aware scheduling with a fully analytical batch processing time model~\citep{bari2025optimal}. Our work differs in focus: we characterize the stability region, identify work-conserving scheduling algorithms as a fundamental design principle for throughput optimality, and extend the analysis to networked LLM systems serving AI-agent workloads.

%% file: sections/sec2_llm_inference_setup.tex
\section{LLM Inference Setup}\label{sec:model}
In this section, we provide a brief overview of the LLM inference process and introduce the key terminology used throughout the paper. For a more comprehensive treatment, we refer readers to \cite{hashimoto2025cs336}. At a high level, LLM inference can be modeled as a queueing system in which incoming jobs (requests) exhibit the following distinguishing characteristics:
\begin{enumerate}
\item Each request is processed in two sequential phases—\emph{prefill} followed by \emph{decode}—and computation occurs at the token level.
\item At each iteration, multiple requests may be batched and processed jointly at the token level, subject to feasibility constraints imposed by system resources and practical considerations.
\end{enumerate}

\subsection{Two-Phase LLM Inference}
\label{sec:model-setup}

We consider an LLM server that hosts an LLM on one or more GPUs (e.g., via tensor parallelism). Throughout this paper, we use \emph{LLM server} in the classical queueing-theoretic sense: the hardware and software unit that processes requests (the ``server'' in the queue). The broader \emph{LLM inference system} encompasses the scheduling algorithms, KV-cache management, and other software components that govern how the LLM server is utilized; our focus is on the scheduling algorithms within such a system. Incoming requests arrive continuously, each associated with an input prompt (e.g., ``Is Paris a city?''). The LLM processes each request at the token level, where each token corresponds to a word, subword, or character in the text.

As shown in \cref{fig:llm_inference}, the inference process executes the request in strict token order through a series of iterations, each corresponding to one forward pass of the LLM model. The overall procedure consists of two phases:
\begin{itemize}
\item Prefill phase: The model processes the input tokens---also referred to as prefill tokens---to initialize its internal key-value (KV) cache. Since all prefill tokens are known in advance, they can be processed jointly within a single forward pass, or split across multiple forward passes as chunks, for computational efficiency.
\item Decode phase: After the prefill phase, the model begins autoregressive decoding. At each iteration, it generates one new token while reusing the KV cache from all previous tokens and producing new KV entries for the newly generated token. This process continues until the model emits an end-of-sequence token or a predetermined stopping condition is met (e.g., a maximum token budget).
\end{itemize}

In essence, a request can be represented as a sequence of tokens
\[
(x_1, x_2, \dots, x_p, x_{p+1}, \dots, x_{p+d}),
\]
where the first \(p\) tokens correspond to the \textit{prefill tokens} and the remaining \(d\) tokens correspond to the \textit{decode tokens}. 
The model must process these tokens in order.
At each iteration, the model can process a consecutive chunk of tokens,
subject to the constraint that the chunk size is \(1\) for the decode tokens—
that is, only one new token can be processed (or generated) at a time.

Formally, if an iteration processes a chunk of tokens 
\((x_{l}, x_{l+1}, \dotsc, x_{r})\) for some \(1 \leq l \leq r \leq p + d\) 
from a given request, we impose the following \textit{feasibility constraints}:
\begin{enumerate}[label=(\alph*)]
\item \((x_{1}, x_{2}, \dotsc, x_{l-1})\) have already been processed in previous iterations.
\item If any token in \((x_{l}, x_{l+1}, \dotsc, x_{r})\) belongs to the decode phase, then the chunk must contain exactly one token, i.e., \(l = r\).
\end{enumerate}

\begin{figure}
\centering
\includegraphics[width=0.9\linewidth]{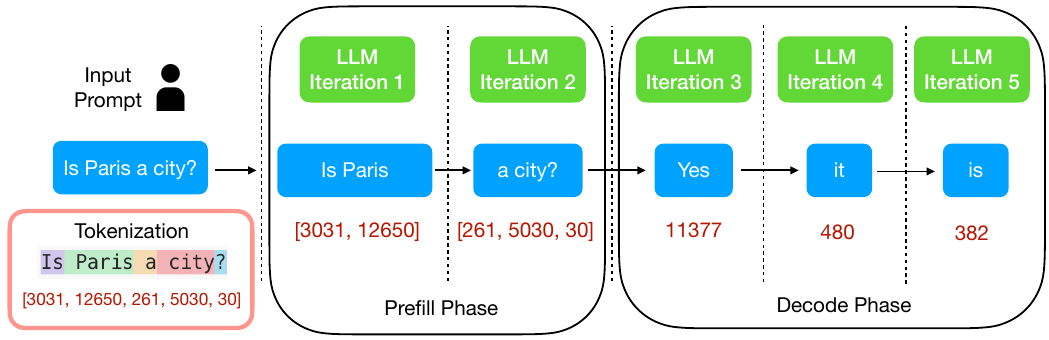}
\caption{Illustration of the LLM inference process, showing tokenization, prefill phase, and decode phase. The tokenization is performed using the GPT-4o tokenizer \citep{openai_tokenizer}.}
\label{fig:llm_inference}
\end{figure}

\subsection{Iteration-Level Batch Serving}
\label{sec:iteration-batch}

While the previous discussion focuses on how a single request is processed by an LLM, 
in practice, an LLM inference system must handle many concurrent requests efficiently. 
To achieve this, during each iteration, the scheduler collects a \textit{batch} of tokens—
each drawn from potentially different requests—and processes them together in a single forward pass 
to maximize GPU utilization.

Formally, let \(\mathcal{B}\) denote the set of requests selected for processing in the current batch. Note that not all unfinished requests need to be included; the scheduler chooses \(\mathcal{B}\) at each iteration. A batch can then be represented as
\(\{(x^{i}_{l_i}, x^{i}_{l_i+1}, \dotsc, x^{i}_{r_i}) : i \in \mathcal{B}\}\),
where, for each request \(i \in \mathcal{B}\), $l_i$ denotes its first unfinished token index and the interval \((l_i, r_i)\) must satisfy the feasibility constraints
defined in \cref{sec:model-setup}.

We define the total number of tokens included in a batch as the \textit{token load} \(b\):
\begin{align}
b := \sum_{i \in \mathcal{B}} (r_i - l_i + 1).
\end{align}
As we will see momentarily in \cref{sec:batch-processing-time}, 
the token load \(b\) serves as a key quantity in characterizing the processing time of each iteration.

\textbf{Token Budget Constraint.}  In practice, the token load \(b\) cannot grow without bound due to hardware constraints. Following Sarathi-Serve \citep{agrawal2023sarathi}, we impose a \textit{token budget} \(b_{\max}\) that caps the number of tokens processed per iteration:
\begin{align}
  b \le b_{\max}, \label{eq:token-budget-bound}
\end{align}
thereby bounding iteration latency and ensuring that the decode phase is not blocked by an excessively large token load. Typical configurations select \(b_{\max} \in \{2^{s}, s\in \mathbb{N}\}\); for example, $b_{\max}=256,512,1024,2048$ have been evaluated in Sarathi-Serve.
\Cref{fig:scheduling-model-overview} illustrates an example of iteration-level batch serving.

\begin{figure}
  \centering
  \includegraphics[width=\linewidth]{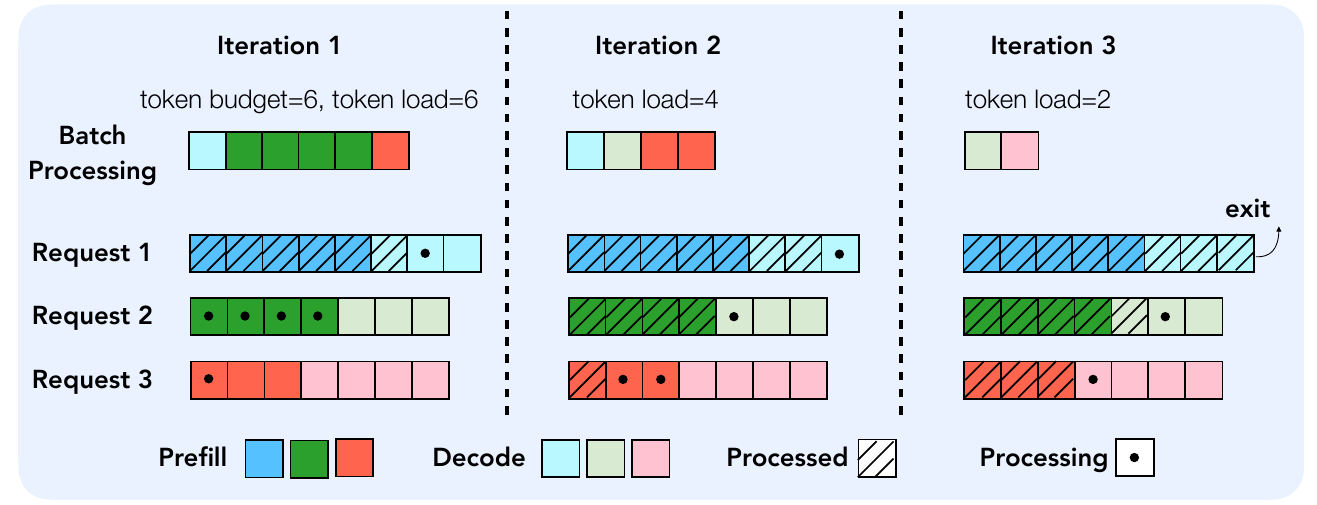}
  \caption{
      Illustration of the iteration-level batch serving.
      In this example, prior to Iteration~1, Request~1 has entered the decoding stage,
      while Requests~2 and~3 have just arrived at the LLM server.
  }
  \label{fig:scheduling-model-overview}
\end{figure}

\textbf{Batch Size Constraint.}
We define the \textit{batch size} \(k := |\mathcal{B}|\) as the number of requests in the current batch.
In practice, LLM inference systems typically impose an upper bound
\begin{align}
  k \le k_{\max} \label{eq:batch-size-constraint}
\end{align}
to control GPU memory usage and to prevent excessive overhead from too many concurrent requests.
We refer to \(k_{\max}\) as the \textit{maximum batch size}.
We omit this constraint in our main analysis to highlight the core insights,
but revisit its impact in \cref{sec:extension-to-active-request-constraint}
and leave a detailed treatment for future work.

Both constraints are exposed in practical LLM inference systems: in vLLM-V1 \citep{vllmv1}, \(b_{\max}\) and \(k_{\max}\) can be configured via \texttt{--max-num-batched-tokens} and \texttt{--max-num-seqs}, respectively; in SGLang \citep{sglang2025}, the corresponding flags are \texttt{--chunked-prefill-size} and \texttt{--max-running-requests}.

We summarize key LLM inference terminologies in \cref{tab:llm-inference-terminologies}.

\begin{table}[ht]
  \centering
  \begin{tabular}{@{}ll@{}}
      \toprule
      \textbf{Term} & \textbf{Definition} \\
      \midrule
      LLM Server & A server that hosts an LLM \\ 
      & and processes inference requests. \\
      \midrule
      Request & A prompt submitted to the LLM \\ 
      & server for processing. \\
      \midrule
      Token & A minimal textual unit (word, \\ 
      & subword, or character) processed \\ 
      & by the LLM. \\
      \midrule
      Prefill Phase & The phase where prefill tokens \\ 
      & are processed to initialize the KV cache. \\
      \midrule
      Decode Phase & The phase where decode tokens \\ 
      & are generated autoregressively. \\
      \midrule
      Iteration & One forward pass of the model. \\
      \midrule
      Batch (\(\mathcal{B}\)) & The set of requests selected for processing \\
      & together in one iteration. \\
      \midrule
      Token Load (\(b\)) & Total number of tokens included in a batch. \\
      \midrule
      Token Budget (\(b_{\max}\)) & Maximum number of tokens allowed in a batch. \\
      \midrule
      Batch Size (\(k\)) & Number of requests processing in a batch. \\
      \midrule
      Maximum Batch Size (\(k_{\max}\)) & Maximum number of requests allowed in a batch. \\
      \bottomrule
  \end{tabular}
  \caption{Summary of key terminologies used to describe the LLM inference process.}
  \label{tab:llm-inference-terminologies}
\end{table}

\subsection{Batch processing time} 
\label{sec:batch-processing-time}
Batch processing time---the time required to execute a selected batch---is a key quantity that governs both latency and throughput in LLM inference.
In principle, it can depend on many factors, including the hardware device, the system-level implementation, and the concrete constituents of the batch.
However, as \cref{fig:batch-time-independent-constituents} demonstrates, in realistic traces the batch processing time varies little across different batch constituents when the token load $b$ is held fixed. This observation motivates a simple yet accurate analytical model.

\begin{figure}[ht]
  \centering
  \includegraphics[width=0.55\linewidth]{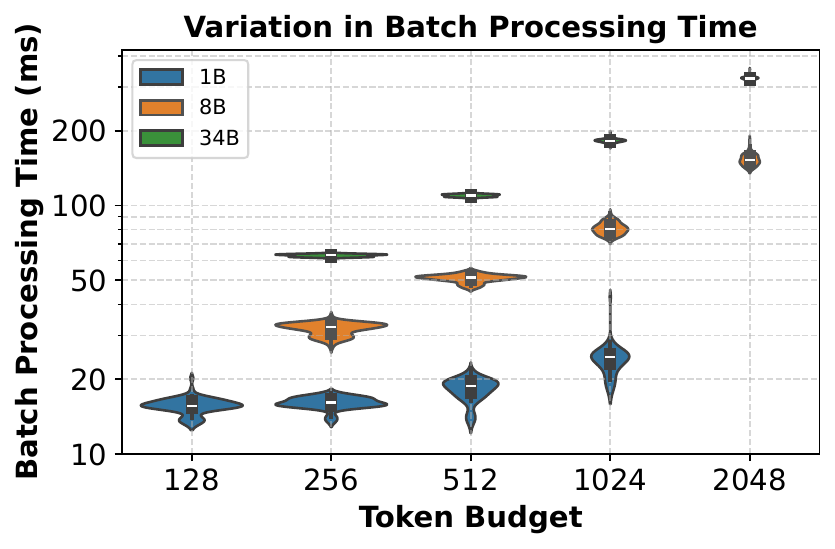} \vspace{-1.2em}
  \caption{Batch processing time remains relatively constant for a given token load (when the LLM is at full token load), and the coefficient of variation becomes even smaller for larger models. The time is measured using the SGLang inference engine under high load with various token loads for Llama series models, with request prefill/decode compositions driven by the ShareGPT dataset~\citep{sharegpt}.}
  \label{fig:batch-time-independent-constituents}
\end{figure}

While machine-learning models can predict batch processing time \citep{agrawal2024vidur}, they are typically hard to analyze and offer limited insight into scheduling-algorithm design.
Leveraging the empirical observation above, we approximate batch processing time by a \textit{piecewise constant} function that depends only on the token load \(b\)\footnote{An earlier version of this manuscript used the approximation \(t_b = c + a\,\max(b-b_0,0)\). The updated model provides a better fit to measured batch times and does not complicate the analysis.}:
\begin{align}
    t_{b} = c + a \cdot \lceil \frac{b}{b_0} \rceil. \label{eq:batch-processing-time}  
\end{align}

\begin{figure}[ht]
  \centering
  \begin{subfigure}{0.49\linewidth}
      \centering
      \includegraphics[width=\linewidth]{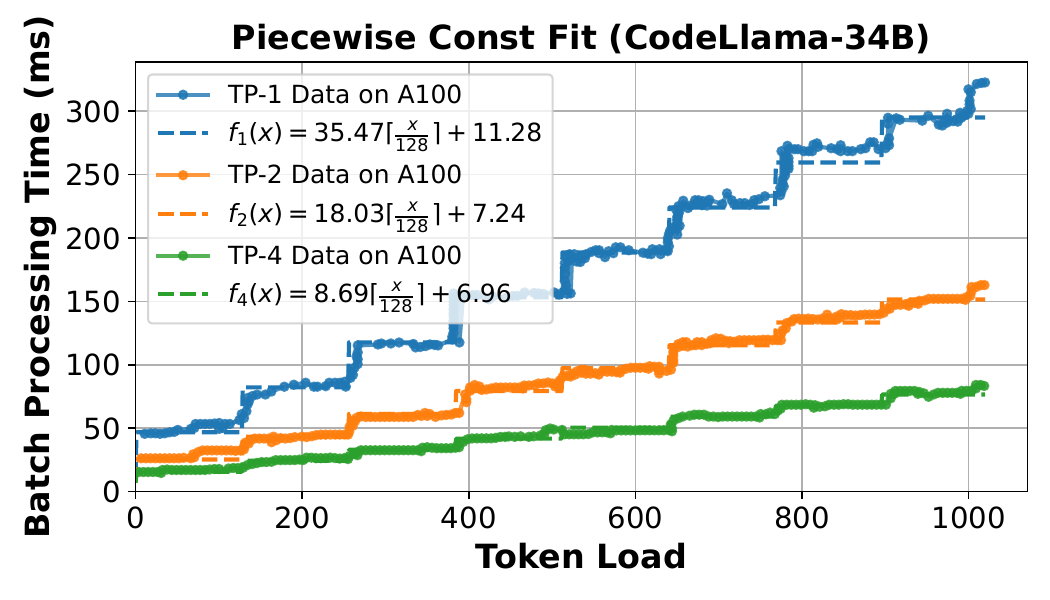}
      \caption{CodeLlama-34B (TP-1, TP-2, TP-4)}
      \label{fig:llama34b}
  \end{subfigure}
  \hfill
  \begin{subfigure}{0.49\linewidth}
      \centering
      \includegraphics[width=\linewidth]{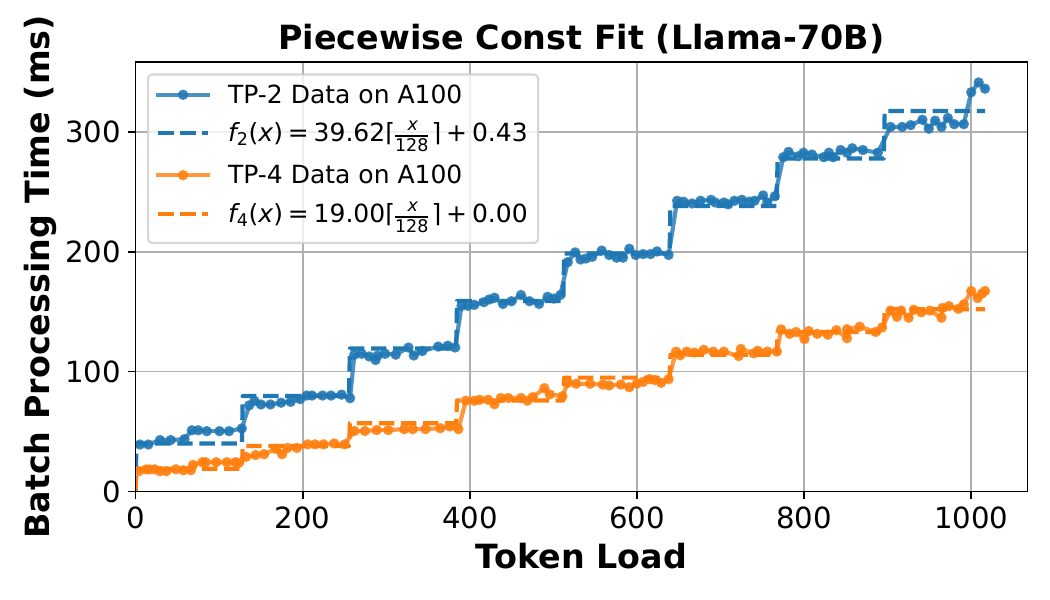}
      \caption{Llama-70B (TP-2, TP-4)}
      \label{fig:llama70b}
  \end{subfigure}
  \caption{Piecewise constant fit of batch processing time for CodeLlama-34B and Llama-70B models under various token loads and tensor-parallel (TP) sizes using the form $t_b = c + a \cdot \lceil b / b_0 \rceil$ with $c \geq 0$. $R^2$ is above 0.97 for all cases.}
  \label{fig:batch-processing-time-linear-form}
\end{figure}

\Cref{fig:batch-processing-time-linear-form} shows that this piecewise constant model closely matches measured batch processing times across a range of settings, with $R^2$ above 0.97 in all cases. The parameters \((c,a,b_0)\) are determined by the serving configuration and treated as fixed. For example, for CodeLlama-34B on a single A100 80\,GB GPU (TP-1), we observe \(c=11.28\)\,ms, \(a=35.47\)\,ms, and \(b_0=128\) tokens; with two GPUs under tensor parallelism (TP-2), the parameters become \(c=7.24\)\,ms and \(a=18.03\)\,ms, and with four GPUs (TP-4), \(c=6.96\)\,ms and \(a=8.69\)\,ms, revealing that the piecewise constant relationship continues to hold when serving across multiple GPUs with tensor parallelism.

The strong agreement between \cref{eq:batch-processing-time} and measured processing times stems from the dominance of linear-layer computations (i.e., feed-forward layers) in typical LLM inference workloads \citep{kamath2024pod,zhu2024nanoflow,ye2025flashinfer}. This approximation is most accurate when request lengths are not extremely long and batches are not extremely large, which covers realistic serving traces (e.g., ShareGPT~\citep{sharegpt}). We emphasize that \cref{eq:batch-processing-time} is not intended to capture the full complexity of LLM batch-based inference; in \cref{sec:discussions}, we extend this formula to account for attention-layer computation (which depends on the number of KV cache tokens) and characterize when the approximation breaks down. Rather, it provides a tractable starting point that already enables the throughput-optimality analysis developed in this paper. We welcome further studies on more expressive batch processing time models---for example, incorporating attention-layer costs through the total number of KV cache tokens as in \cite{ao2025optimizing}, or fully analytical batch processing time models as in \cite{bari2025optimal}---and the deeper scheduling-theoretic questions they may unlock.

For the theoretical analysis, we model batch processing times using \cref{eq:batch-processing-time} under the following assumption. In our experiments, we use measured batch processing times directly.
\begin{assumption}[Batch Processing Time]\label{asm:batch-processing-time}
    Batch processing times follow \cref{eq:batch-processing-time}. Furthermore, we assume that $a, c\geq 0$ and $b_{\max}$ is an integer multiple of $b_0$.
\end{assumption}
These conditions are mild and motivated by practical considerations:
\begin{itemize}
    \item $a, c \geq 0$: The marginal cost $a$ per additional block of $b_0$ tokens and the base cost $c$ per iteration are both nonnegative, which holds in all practical settings.
    \item $b_0 \mid b_{\max}$: The parameter $b_0$ reflects the hardware-aligned block size (e.g., GPU matrix-multiplication tiles), typically a power of two. The token budget $b_{\max}$ is naturally chosen as a multiple of this block size, often also a power of two.
\end{itemize}

%% file: sections/sec3_stochastic_modeling.tex
\section{Stochastic Modeling and Throughput-Optimal Scheduling Algorithms}
Throughout the paper, we focus on \textit{batch scheduling algorithms}---policies that select the batch constituents at each iteration. Given the LLM inference setup, we now study the following question:
\begin{quote}
\it{What is the maximal throughput of an LLM inference system, and which batch scheduling algorithms can achieve it?}
\end{quote}
Questions of this form are foundational in queueing theory: for complex queueing systems, throughput is often the first performance objective to characterize, after which one can study latency and related metrics. We follow this tradition for LLM inference systems.

\textbf{A new queueing problem.}
From a queueing-theoretic perspective, LLM inference gives rise to a novel two-class batch-service system (see \cref{fig:two-class-queue}). Each request first joins the \emph{prefill queue} and, upon completion, transitions to the \emph{decode queue}. The server processes a batch of tokens drawn from both queues simultaneously, subject to the token budget $b_{\max}$. Two features distinguish this system from classical batch-service queues:
\begin{enumerate}
    \item \textbf{Cross-class batching.} A single batch can mix prefill tokens and decode tokens from different requests, coupling the two queues through a shared capacity constraint.
    \item \textbf{Asymmetric per-job processing.} In the prefill phase, multiple tokens from the same request can be processed in one batch; in the decode phase, each request contributes at most one token per batch. This asymmetry is a direct consequence of the autoregressive nature of decoding.
\end{enumerate}
These features make LLM inference a unique batch-service system and present an interesting new application challenge for the queueing community.

\begin{figure}[t]
\centering
\begin{tikzpicture}[
    >=Stealth,
    queue/.style={draw, minimum width=2.6cm, minimum height=0.7cm, rounded corners=2pt, font=\small},
    server/.style={draw, thick, rounded corners=4pt, fill=gray!6},
    token/.style={draw, fill=#1!25, minimum size=0.35cm, inner sep=0pt},
    arr/.style={->, thick},
    lbl/.style={font=\footnotesize},
    every node/.style={font=\small}
]

\node[queue, fill=blue!10] (pq) at (-4.2, 0.8) {Prefill queue};
\node[queue, fill=orange!10] (dq) at (-4.2,-0.5) {Decode queue};

\draw[arr, purple!70, dashed] (-4.2, 0.4) -- (-4.2,-0.1)
    node[midway, left=2pt, font=\footnotesize] {prefill complete};

\draw[arr] (-6.5,0.8) -- (pq.west) node[midway, above, lbl] {arrivals};

\draw[arr] (-4.2,-0.9) -- (-4.2,-1.5) node[below, lbl] {departure};

\node[server, minimum width=5.8cm, minimum height=3.4cm] (srv) at (1.8,0) {};
\node[font=\small\bfseries] at (1.8, 2.05) {GPU (one batch)};

\draw[gray, densely dotted] (-0.9, 0.0) -- (4.5, 0.0);
\node[font=\scriptsize, gray, right] at (3.6, 0.15) {prefill};
\node[font=\scriptsize, gray, right] at (3.6,-0.18) {decode};

\foreach \x in {0,1,2,3,4} {
    \node[token=blue] at (-0.4+\x*0.5, 1.05) {};
}
\node[lbl, right] at (1.7, 1.05) {req $j_1$\; (5 tokens)};

\foreach \x in {0,1,2} {
    \node[token=blue] at (-0.4+\x*0.5, 0.45) {};
}
\node[lbl, right] at (1.2, 0.45) {req $j_2$\; (3 tokens)};

\node[token=orange] at (-0.4,-0.55) {};
\node[lbl, right] at (-0.05,-0.55) {req $i_1$};

\node[token=orange] at (1.4,-0.55) {};
\node[lbl, right] at (1.75,-0.55) {req $i_2$};

\node[token=orange] at (-0.4,-1.15) {};
\node[lbl, right] at (-0.05,-1.15) {req $i_3$};

\node[token=orange] at (1.4,-1.15) {};
\node[lbl, right] at (1.75,-1.15) {req $i_4$};

\draw[decorate, decoration={brace, amplitude=5pt}]
    (4.85, 1.35) -- (4.85, 0.15)
    node[midway, right=6pt, lbl, align=left] {multiple tokens\\per request};
\draw[decorate, decoration={brace, amplitude=5pt}]
    (4.85,-0.15) -- (4.85,-1.4)
    node[midway, right=6pt, lbl, align=left] {$\le 1$ token\\per request};

\draw[arr, blue!60] (pq.east) -- (-1.1, 0.55);
\draw[arr, orange!60] (dq.east) -- (-1.1,-0.75);

\node[font=\footnotesize] at (1.8,-2.0) {token budget: $\sum\text{tokens} \le b_{\max}$};

\end{tikzpicture}
\caption{LLM inference as a two-class batch-service queueing system. Arriving requests join the prefill queue; upon completing prefill, they transition to the decode queue. The GPU processes a mixed batch from both queues, subject to the token budget $b_{\max}$. Crucially, multiple prefill tokens from the same request can be batched together, whereas each decode request contributes at most one token per batch.}
\label{fig:two-class-queue}
\end{figure}

\subsection{Basic LLM queueing model}\label{sec:markov-chain}
To study this problem rigorously, we introduce a stochastic model of the inference process, which we refer to as the \emph{basic LLM queueing model}. Although the notation is somewhat heavy, it enables precise analysis and provides a systematic, broadly applicable framework for reasoning about the system. We will use the term \emph{basic LLM queueing model} throughout the paper to distinguish this single-server, single-class formulation from the multi-class and multi-server extensions developed in \cref{sec:AI-agent}.

For simplicity, we assume time is discrete, indexed by $n\in \N\equiv \{1,2,\ldots\}$. A time slot can correspond to the physical clock frequency of the system; extending the analysis to continuous time is straightforward.

\textbf{Arrival.} Multiple requests may arrive in the same time slot, and we index them in order of arrival. For $i=1,2,\ldots$, request $i$ has a prefill length $v_p(i)\in\N$ and a decode length $v_d(i)\in\N$. Since every LLM has a finite context window and a maximum output length, the token sizes are bounded: $v_p(i)\le v_p^{\max}$ and $v_d(i)\le v_d^{\max}$ for all $i$, where $v_p^{\max}$ and $v_d^{\max}$ are finite constants. We assume the sequence $\{(v_p(i),v_d(i)) : i\in\N\}$ is i.i.d.\ and define
\[
m_p:=\E[v_p(1)],\qquad m_d:=\E[v_d(1)].
\]
In addition, let $a_n$ denote the number of requests arriving in time slot $n$. We assume $\{a_n:n\in\N\}$ is an exogenous i.i.d.\ sequence, independent of $\{(v_p(i),v_d(i)) : i\in\N\}$, and define the arrival rate
\[
\lambda := \E[a_1].
\]
It follows that the expected prefill-token workload arriving per time slot is \(\lambda m_p\), and the expected decode-token workload arriving per time slot is \(\lambda m_d\).

\textbf{State.} With the arrival setup complete, we introduce a discrete-time Markov chain (DTMC) to describe the system dynamics. We begin by defining the system state.

At each time slot \(n\), let \(\mathcal{Q}_n\) denote the set of requests that have not yet departed. For each request \(i\in\mathcal{Q}_n\), let \(P_i(n)\) and \(D_i(n)\) denote the number of \textbf{unprocessed tokens} in the prefill and decode stages, respectively, at the beginning of slot \(n\). Thus, \(\{(P_i(n), D_i(n)) : i \in \mathcal{Q}_n\}\) captures the status of all unfinished requests.  

Since batch processing may span multiple time slots (which we assume are integer-valued), let \(R(n)\) denote the \textbf{remaining processing time} (in slots) of the current batch. If \(R(n)=0\), a new batch can be formed; otherwise, no decision is required. The system state at the beginning of slot \(n\) is then defined as
\begin{align}
  \label{eq:state}
   X(n) =\Big(R(n), \{(P_i(n), D_i(n)) : i \in \mathcal{Q}_n\}\Big).
\end{align}  
Let \(\mathcal{X}\) denote the set of all possible states \(X(n)\) for \(n\in\N\) (i.e., the DTMC state space).

\textbf{Scheduling algorithms}. Let \( x = (r, \{(p_i, d_i) : i \in \mathcal{Q} \}) \in \mathcal{X} \) be an arbitrary system state, where $r$ is the remaining processing time of the current batch, and $p_i$, $d_i$ are the unprocessed prefill and decode token counts of request $i$. Given \( x \) with \( r = 0 \), a scheduling algorithm selects a batch configuration \(\pi(x) = (\delta_i^p, \delta_i^d)_{i \in \mathcal{Q}}\), where \( \delta_i^p \) and \( \delta_i^d \) denote the number of prefill tokens and decode tokens, respectively, from request \( i \) to be included in the batch. Unlike the $(l_i, r_i)$ representation in \cref{sec:iteration-batch}, which specifies the token range, the $(\delta_i^p, \delta_i^d)$ characterization tracks only the number of tokens allocated to each request, which is more convenient for the stochastic analysis. For \( \pi(x) \) to be feasible, it must satisfy:  
\begin{subequations}\label{eq:feasibility-constraint}
\begin{align}
    & p_{i} > 0 \text{ implies } \delta_{i}^{d} = 0 \label{eq:prefill-pre-decode}\\
    & \delta_{i}^{p}  \leq p_{i} \label{eq:prefill-unprocess} \\ 
    & \delta_{i}^{d} \leq 1  \label{eq:decode-unprocess}\\
    & \sum_{i \in \mathcal{Q}} \Big(\delta_{i}^{d} + \delta_{i}^{p}\Big) = b \leq b_{\max}. \label{eq:batch-constraint}
\end{align}
\end{subequations}
Constraint \cref{eq:prefill-pre-decode} ensures that prefill tokens are completed before decoding begins, while \cref{eq:prefill-unprocess} ensures the number of batched prefill tokens does not exceed the remaining unprocessed ones. \cref{eq:decode-unprocess} enforces sequential decoding (at most one token per request per batch), and \cref{eq:batch-constraint} enforces the total token budget \(b_{\max}\). When \( r > 0 \), we set \( b = 0 \) as batch processing is already ongoing. However, even when \( r = 0 \), an idle slot (i.e., \( b = 0 \)) is permitted. Any policy \(\pi\) satisfying conditions 
\cref{eq:prefill-pre-decode}-\cref{eq:batch-constraint} is referred to as a (batch) scheduling algorithm. Note that the feasibility constraints \cref{eq:feasibility-constraint} correspond to those introduced in \cref{sec:model-setup}, with the token budget $b_{\max}$ enforced here; the batch size constraint $k \le k_{\max}$ is deferred to \cref{sec:extension-to-active-request-constraint}.

Note that when making decisions, the scheduling algorithm $\pi$ can utilize the index $i$ of each unfinished request. This index may encode request-specific characteristics—such as type, content, or arrival order—thereby encompassing a broad range of potential scheduling strategies. 

\textbf{System dynamics.} Let $X(n)=(R(n), \{(P_i(n), D_i(n)) : i\in \mathcal{Q}_n\})$ be the system state at the beginning of time slot $n$. The system evolves according to the following sequence of events:
\begin{enumerate}
\item \textbf{Scheduling:} If $R(n) = 0$, the algorithm selects a batch configuration $\pi(X(n)) = (\delta_i^p, \delta_i^d)_{i \in \mathcal{Q}_n}$ with token load $b = \sum_{i \in \mathcal{Q}_n}(\delta_i^{p} + \delta_i^{d})$.  We define post-scheduling unprocessed token counts $P_{i}'(n) = P_{i}(n) - \delta_{i}^{p}$ and $D_{i}'(n) = D_{i}(n) - \delta_{i}^{d}$. If $R(n) > 0$, the previous batch continues, and we set $P_i'(n) = P_i(n)$, $D_i'(n) = D_i(n)$.

\item \textbf{Processing:} The remaining processing time $R(n+1)$ for the next slot is updated as:
\begin{align}
    \label{eq:r-transition}
    R(n+1) = \begin{cases}
        R(n) - 1  & \text{if } R(n) > 0, \\
        t_{b} - 1 & \text{if } R(n) = 0 \text{ and } t_{b} > 0, \\ 
        0         & \text{if } R(n) = 0 \text{ and } t_{b} = 0,
    \end{cases}
\end{align}
where $t_b$ is the number of slots required to process a batch of token load $b$.

\item \textbf{Departures:} A batch finishes processing at the end of slot $n$ if its remaining time $R(n+1)$ reaches zero (provided the slot was not idle, i.e., $R(n) > 0$ or $t_b > 0$). When a batch completes, requests with no remaining tokens depart, and the set of requests becomes $\mathcal{Q}_{n}' = \mathcal{Q}_n \setminus \{i\in \mathcal{Q}_n: D'_i(n)=0\}$. Otherwise, $\mathcal{Q}_n' = \mathcal{Q}_n$.

\item \textbf{Arrivals:} New requests $\mathcal{A}_n$ arrive. The state at the beginning of slot $n+1$ is $X(n+1)=(R(n+1), \{(P_i(n+1), D_i(n+1)) : i\in \mathcal{Q}_{n+1}\})$, where $\mathcal{Q}_{n+1} = \mathcal{Q}'_n \cup \mathcal{A}_n$. For $i\in \mathcal{Q}_n'$, token counts are $P_i(n+1) = P_i'(n)$ and $D_i(n+1) = D_i'(n)$; for new arrivals $i\in \mathcal{A}_n$, they are $P_i(n+1) = v_p(i)$ and $D_i(n+1) = v_d(i)$.
\end{enumerate}
Given the system dynamics above and the i.i.d.\ assumptions on arrivals, the sequence $\{X(n) : n\in \N\}$ forms a  DTMC in a countable state space.

\subsection{Throughput-optimal scheduling algorithms}

After formalizing the scheduling problem in \cref{sec:markov-chain}, we are now in a position to analyze the system rigorously. 

A scheduling algorithm is said to achieve throughput $\lambda$ if the associated DTMC $\{X(n), n \in \mathbb{N}\}$ is irreducible and positive recurrent when the arrival rate is $\lambda$; see Section~5.7 of \citep{dai2020processing} for a discussion of maximally stable scheduling algorithms. Our main result is that a class of \emph{work-conserving} scheduling algorithms can (almost) achieve the system’s maximal throughput rate.

Specifically, a scheduling algorithm $\pi$ is said to be \textit{work-conserving} if $\pi(x)$ forms a batch of token load $b = b_{\max}$ whenever possible. Namely, $\pi(x)$ satisfies
\begin{align}
\sum_{i\in \mathcal{Q}}\Big( \delta_{i}^{d} + \delta_{i}^{p}\Big) = b_{\max} 
\end{align}
whenever
\begin{align}
\sum_{i\in \mathcal{Q}} \Big( p_{i} +  \boldsymbol{1}(p_{i} =0)\Big)  \ge b_{\max}.
\end{align}
Here, the left-hand side represents the maximum number of tokens that could be included in a batch: each request in the prefill phase ($p_i > 0$) can contribute up to $p_i$ tokens, while each request in the decode phase ($p_i = 0$) can contribute at most $1$ token. Whenever this supply meets or exceeds $b_{\max}$, a work-conserving algorithm must fill the batch to the full token budget, by allowing batches that mix prefill and decode tokens.

Before stating the theorem, we introduce one additional condition on the scheduling algorithm that facilitates the stability analysis via the fluid limit technique~\citep{dai1995positive,stolyar1995stability,dai2020processing}. This technique is also used to prove stability results for networks of LLM servers in \cref{sec:AI-agent}.

\noindent \textbf{$K$-FCFS condition.}
Index requests as $i = 1, 2, \dotsc$ according to their global arrival time. A work-conserving algorithm satisfies the \emph{$K$-FCFS} condition if,
for any request $j$ that consumes at least one decoding token in the batch
(i.e., $\delta_{j}^d > 0$), the following holds:
\[
\forall\, i \le j-K,~\text{if request $i$ is still in the system, then }
\delta_{i}^p + \delta_{i}^d \;\ge\; 1.
\]
In other words, if a request that arrived at least $K$ positions later than request $i$ processes a decoding token, then $i$ must also process at least one token in the same batch. The condition is triggered by \emph{decoding} tokens ($\delta_j^d > 0$) rather than by any token processing; this accommodates prefill-prioritized algorithms that may process a batch of prefill tokens for newly arrived requests before older requests resume decoding (e.g., see \cref{sec:incumbent}). 

This is a mild condition: no practical system would let a request wait indefinitely while serving newer arrivals. Most scheduling policies (e.g., FCFS-based) naturally satisfy this condition for a reasonable $K$. The exact value of $K$ is not important in our analysis, as long as it is finite.

We now state the main theorem.
\begin{theorem}\label{thm:main}
  Under the i.i.d.\ arrival model in \cref{sec:markov-chain} and \cref{asm:batch-processing-time}:

\noindent (a) \textbf{(Instability.)} If the system load satisfies
   \begin{align}\label{eq:load>}
  \lambda (m_p+m_d) > b_{\max} / t_{b_{\max}},
   \end{align}
then the total number of unprocessed tokens diverges under \emph{any} scheduling algorithm:
   \begin{align}
     \label{eq:diverge}
     \Prob\{ \lim_{n\to \infty} \abs{X(n)} =\infty\}=1,
   \end{align}
   where $\abs{X(n)}=\sum_{i\in \mathcal{Q}_n}(P_i(n)+D_i(n))$.

\noindent (b) \textbf{(Stability.)} If the system load satisfies
  \begin{align}\label{eq:load<}
  \lambda (m_p+m_d) < b_{\max} / t_{b_{\max}},
\end{align}
then the DTMC $\{X(n), n\in \N\}$ is positive recurrent (i.e., the system is stable) under any work-conserving $K$-FCFS algorithm.
\end{theorem}

\Cref{thm:main} provides a complete characterization of the stability region for scheduling algorithms subject to the token budget constraint: the critical threshold is $\lambda(m_p+m_d) = b_{\max}/t_{b_{\max}}$, where $b_{\max}/t_{b_{\max}}$ is the maximal token processing rate (see the proof sketch below). The theorem assumes deterministic batch processing times for clarity; the result extends easily to random processing times by replacing $t_b$ with $\E[t_b]$ throughout.

This result has several practical implications:
\begin{enumerate}
    \item \textbf{Guideline for scheduling policy design.} Practitioners should adopt work-conserving scheduling algorithms. In practice, the LLM inference community has been moving from non-work-conserving schedulers (e.g., FasterTransformer, vanilla vLLM) toward work-conserving ones (e.g., Sarathi-Serve, vLLM with chunked prefill) (see \cref{sec:incumbent} for a detailed examination). \Cref{thm:main} provides a rigorous theoretical basis to endorse this trend and to guide future policy design.

    \item \textbf{Throughput benchmark.} The quantity $b_{\max}/t_{b_{\max}}$ represents a theoretical upper bound on the token processing rate. It serves as an ideal-case benchmark against which practitioners can measure and optimize their system's throughput.

    \item \textbf{Foundation for complex systems.} \Cref{thm:main} establishes the baseline for a single LLM server, providing a foundation for deeper investigation. In \cref{sec:AI-agent}, we build on this result to analyze networks of LLM servers handling AI-agent workloads, identifying conditions under which work-conserving algorithms remain throughput-optimal and conditions under which they fail.
\end{enumerate}

Next we present a proof sketch; see \cref{sec:basic} for the full fluid model proof.

\begin{proof}[Proof Sketch.]
\textbf{Part~(a): Instability.}
We show that $b_{\max}/t_{b_{\max}}$ is the maximal token processing rate achievable by any algorithm. Under \cref{eq:batch-processing-time}, the rate $b/t_b = b/(c+a\lceil b/b_0\rceil)$ is maximized at the right endpoint of each step $b = kb_0$, where it equals $kb_0/(c+ak)$. Since $\frac{d}{dk}\frac{kb_0}{c+ak} = \frac{b_0\,c}{(c+ak)^2}>0$, these step-endpoint rates are strictly increasing in $k$, so the overall maximum is $b_{\max}/t_{b_{\max}}$ (under $b_0 \mid b_{\max}$ from \cref{asm:batch-processing-time}).
Hence, at most $b_{\max}/t_{b_{\max}}$ tokens can be processed per unit time under any algorithm. By the strong law of large numbers, the cumulative arriving workload grows at rate $\lambda(m_p+m_d) > b_{\max}/t_{b_{\max}}$, so the total unprocessed tokens $|X(n)| \to \infty$ almost surely. See also the proof of Proposition 5.1 of \citet{Dai1996}.

\textbf{Part~(b): Stability.}
Under the load condition \cref{eq:load<}, we construct a Lyapunov function on the fluid model and show the drift is negative whenever the system is non-empty, so the fluid model drains in finite time. The non-trivial part is establishing that \emph{every} fluid limit of the stochastic system---obtained as a scaling limit along subsequences of initial states---satisfies the work-conserving fluid model equations. This requires proving uniform convergence on compact sets (u.o.c.\ convergence) of the scaled stochastic processes to their fluid analogs, adapted to the two-class batch-service structure of LLM inference. Once established, fluid model stability implies positive recurrence of the DTMC~\citep{dai1995positive,stolyar1995stability,dai2020processing}. The full proof is given in \cref{sec:basic}.

An alternative, self-contained proof of Part~(b) directly constructs a Foster--Lyapunov function on the DTMC (\cref{sec:proof-thm-main}). This approach is shorter and uses only the work-conserving property (without requiring $K$-FCFS), but is less extensible; the fluid limit technique generalizes naturally to the network settings of \cref{sec:AI-agent}.
\end{proof}

\begin{figure}[t]
    \centering
    \includegraphics[width=\linewidth]{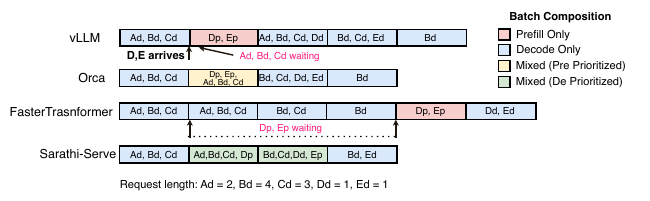}
    \caption{Example workload and where work-conserving criteria are broken. For vLLM, the second batch is not work-conserving because of limited prefill, and decoding tokens from earlier requests are still waiting. For FasterTransformer, the second to fourth batches are not work-conserving, because the prefills are blocked with earlier decodes. }
    \label{fig:sec3-ill}
\end{figure}

\subsection{Extension to parallel servers}
\label{sec:parallel-servers}

The result of \cref{thm:main} extends naturally from a single LLM server to $K$ identical servers operating in parallel behind a load balancer. We consider two common load-balancing algorithms:
\begin{itemize}
    \item \textbf{Random assignment:} each arriving request is routed to a server chosen uniformly at random, so each server receives an effective arrival rate $\lambda_k = \lambda/K$.
    \item \textbf{Join-the-shortest-queue (JSQ):} each arriving request is routed to the server $k$ with the fewest total unfinished requests $\abs{\mathcal{Q}^k}$, with ties broken arbitrarily.
\end{itemize}
The state in \cref{eq:state} needs to be modified as
\begin{align}
  \label{eq:stateK}
   X(n) =\Big(R^k(n), \{(P^k_i(n), D^k_i(n)) : i \in \mathcal{Q}^k_n\},k=1, \ldots, K\Big).
\end{align}  

\begin{proposition}\label{pro:parallel}
  Suppose $K$ identical servers each run a work-conserving $K$-FCFS scheduling algorithm, with requests routed by either random assignment or JSQ. If the aggregate load condition
\begin{align}\label{eq:station_load}
    \lambda (m_p+m_d) < K \cdot b_{\max}/t_{b_{\max}}
\end{align}
holds, then the DTMC $\{X(n), n\ge 0\}$ describing the system is positive recurrent.
\end{proposition}
The proof, which applies the fluid limit technique of \cref{sec:basic}, is given in \cref{sec:proof-pro-parallel}. The argument uses the total system workload across all $K$ servers as a Lyapunov function. For random assignment, each server independently receives a fraction $\lambda/K$ of the arrivals, so each busy server has negative drift by the same single-server argument as \cref{thm:main}. The JSQ case requires a more delicate Lyapunov drift argument that exploits the property that JSQ never sends work to a server that holds more workload in the fluid regime; see \cref{sec:proof-pro-parallel} for details.

\begin{remark}[Other routing metrics]
\label{rem:jsq+}
The result extends well beyond JSQ. It holds for any load-balancing algorithm that routes each arrival to a server minimizing some non-negative scalar measure of server occupancy ---examples include total remaining tokens, total expected remaining workload, or any monotone increasing function of the prefill and decode counts. See \cref{sec:proof-pro-parallel} for details.
\end{remark}

%% file: sections/sec4_work_conservingness.tex
\section{Work-Conservingness of Existing LLM Inference Systems}
\label{sec:incumbent}

In this section, we examine the work-conservingness and stability of widely adopted scheduling algorithms for LLM inference. We show that FasterTransformer and vanilla vLLM are \emph{not} work-conserving---and can therefore be unstable even when the load condition \cref{eq:load<} is satisfied---while Orca and Sarathi-Serve are work-conserving and hence throughput-optimal by \cref{thm:main}. In practice, the LLM inference community has been progressively moving toward work-conserving designs: the latest version of vLLM with chunked prefill enabled is also work-conserving~\citep{vllm2025chunked}.

We now describe each algorithm in turn.

\textbf{FasterTransformer}~\citep{fastertransformer} \textbf{(decode-prioritized, no mixed batching).}
FasterTransformer is a decode-prioritized scheduler without mixed batching. Whenever there are requests in the decoding stage, it batches as many decode tokens as possible (up to the token budget $b_{\max}$) and processes them, leaving requests in the prefill queue untouched. It is \emph{not} work-conserving: prefill tokens wait while the GPU processes decode-only batches.

\textbf{Vanilla vLLM}~\citep{kwon2023efficient} \textbf{(prefill-prioritized, no mixed batching).}
Vanilla vLLM is a prefill-prioritized scheduler without mixed batching. Whenever there are requests in the prefill queue, it processes prefill tokens and ignores any requests in the decoding queue. It is \emph{not} work-conserving: decode tokens remain idle while the GPU processes prefill-only batches.

\textbf{Orca}~\citep{yu2022orca} \textbf{(prefill-prioritized, with mixed batching).}
Orca improves upon vanilla vLLM by allowing decoding requests to be batched together with prefill requests. It first fills in prefill tokens from requests with $p_i > 0$, then packs as many decode tokens as possible until the token budget $b_{\max}$ is reached. Because it mixes both token types in each batch, Orca is work-conserving.

\textbf{Sarathi-Serve}~\citep{agrawal2023sarathi} \textbf{(decode-prioritized, chunked prefill).}
Sarathi-Serve introduces chunked prefill with decode prioritization. When forming a batch, it first fills in as many decode tokens (at most one per request) as possible; when there is additional space, it fills in as many prefill tokens as possible from a minimal number of requests. Sarathi-Serve is work-conserving.

\textbf{Why FasterTransformer and vanilla vLLM fail.}
\Cref{fig:sec3-ill} illustrates the mechanisms. In vanilla vLLM, prefill is prioritized over decode. When incoming prefill requests are consistently short, the scheduler fails to batch decode tokens together, remaining in the memory-bound regime and causing the decode queue to grow without bound. In FasterTransformer, decode is prioritized but mixed batching is absent. When decode requests have token loads below $b_0$ (see \cref{eq:batch-processing-time}), they are processed sequentially without utilizing the GPU's parallel processing capabilities, while incoming prefill requests remain blocked.

Through our analysis, Orca and Sarathi-Serve---the two work-conserving algorithms---should be preferred. Indeed, \cref{fig:instable} confirms this on practical workloads: the work-conserving schedulers achieve higher throughput, while FasterTransformer and vanilla vLLM become unstable even under the same load. Our theoretical analysis provides a rigorous endorsement of this trend and encourages practitioners to design scheduling policies within the class of work-conserving algorithms.

%% file: sections/sec5_ai_agent_workloads.tex
\section{Throughput Optimality for AI-Agent Workloads}
\label{sec:AI-agent}

LLMs are increasingly deployed as autonomous agents that handle complex tasks through sequences of LLM calls, possibly spanning multiple LLM servers~\citep{luo2025autellix}. The execution flow of a single agentic user request can involve planning, tool selection, generation, verification, refinement, and other steps, forming a \emph{workload} that may be a directed acyclic graph (DAG) or contain loops. \Cref{fig:agentic-workflows} illustrates several representative topologies: a sequential chain, a fork-join pattern, and a self-reflection loop. We refer to these collectively as \textbf{AI-agent workloads}.

\begin{figure}[t]
\centering
\resizebox{\linewidth}{!}{%
\begin{tikzpicture}[
    >=Stealth,
    llm/.style={draw, rounded corners=3pt, minimum width=1.2cm, minimum height=0.6cm, fill=blue!10, font=\footnotesize},
    arr/.style={->, thick},
    lbl/.style={font=\scriptsize},
    every node/.style={font=\small}
]

\node[font=\footnotesize\bfseries] at (0, 1.8) {(a) Sequential chain};
\node[llm] (a1) at (-1.8, 0.6) {Retrieve};
\node[llm] (a2) at (0, 0.6) {Summarize};
\node[llm] (a3) at (1.8, 0.6) {Format};
\draw[arr] (-3.0,0.6) -- (a1);
\draw[arr] (a1) -- (a2);
\draw[arr] (a2) -- (a3);
\draw[arr] (a3) -- (3.0,0.6);

\node[font=\footnotesize\bfseries] at (6.5, 1.8) {(b) Fork-join};
\node[llm] (b0) at (4.6, 0.6) {Plan};
\node[llm] (b1) at (6.5, 1.2) {Search};
\node[llm] (b2) at (6.5, 0.0) {Analyze};
\node[llm, minimum width=1.5cm] (b3) at (8.6, 0.6) {Synthesize};
\draw[arr] (3.6,0.6) -- (b0);
\draw[arr] (b0.east) -- (b1.west);
\draw[arr] (b0.east) -- (b2.west);
\draw[arr] (b1.east) -- (b3.west);
\draw[arr] (b2.east) -- (b3.west);
\draw[arr] (b3) -- (9.8,0.6);

\node[font=\footnotesize\bfseries] at (12.3, 1.8) {(c) Self-reflection};
\node[llm] (c1) at (11.2, 0.6) {Generate};
\node[llm] (c2) at (13.4, 0.6) {Evaluate};
\draw[arr] (10.0,0.6) -- (c1);
\draw[arr] (c1) -- (c2);
\draw[arr, dashed, bend right=50] (c2.north) to node[below, lbl] {refine} (c1.north);
\draw[arr] (c2) -- (14.6,0.6);

\end{tikzpicture}%
}
\caption{Representative AI-agent workflow topologies. (a)~A sequential chain of LLM calls (e.g., retrieval-augmented generation). (b)~A fork-join pattern where subtasks are processed in parallel and synchronized. (c)~A self-reflection loop with iterative refinement.}
\label{fig:agentic-workflows}
\end{figure}

From a queueing-theoretic perspective, we model an AI-agent workload as a \emph{multi-class batch-service processing network}. Each node in the workload graph represents a \emph{class} of LLM-level requests---requests that share the same workload logic and statistical characteristics (prefill and decode token sizes). When a request completes processing in one class, it is routed to another class  according to the workload structure, or departs the system. This formulation naturally captures both static DAGs and dynamic workflows, including agentic loops where a request may revisit the same class.

The key distinction from classical stochastic processing networks~\citep{dai2020processing} is twofold. First, requests from different classes at the same server can be \emph{batched together} and processed simultaneously (\cref{fig:multi-class-batch}). Second, each LLM-level request inherits the two-phase (prefill + decode) structure of \cref{fig:two-class-queue}; in a prefill--decode disaggregated setup, the two phases may even be served by different servers.

\begin{figure}[t]
\centering
\begin{tikzpicture}[
    >=Stealth,
    queue/.style={draw, minimum width=2.0cm, minimum height=0.55cm, rounded corners=2pt, font=\footnotesize},
    server/.style={draw, thick, rounded corners=4pt, fill=gray!6},
    token/.style={draw, fill=#1!25, minimum size=0.3cm, inner sep=0pt},
    arr/.style={->, thick},
    lbl/.style={font=\scriptsize},
    every node/.style={font=\small}
]

\node[queue, fill=blue!10] (q1p) at (-4.5, 1.5) {\scriptsize Generate (P)};
\node[queue, fill=blue!20] (q1d) at (-4.5, 0.8) {\scriptsize Generate (D)};
\node[queue, fill=orange!10] (q2p) at (-4.5, -0.2) {\scriptsize Verify (P)};
\node[queue, fill=orange!20] (q2d) at (-4.5, -0.9) {\scriptsize Verify (D)};

\draw[decorate, decoration={brace, amplitude=4pt, mirror}]
    (-5.9, 1.65) -- (-5.9, 0.65) node[midway, left=5pt, font=\scriptsize, align=center] {Class 1};
\draw[decorate, decoration={brace, amplitude=4pt, mirror}]
    (-5.9, -0.05) -- (-5.9, -1.05) node[midway, left=5pt, font=\scriptsize, align=center] {Class 2};

\node[server, minimum width=4.8cm, minimum height=3.0cm] (srv) at (1.2,0.3) {};
\node[font=\small\bfseries] at (1.2, 2.15) {GPU (one batch)};

\draw[gray, densely dotted] (-0.8, 0.3) -- (3.2, 0.3);
\node[font=\scriptsize, gray, right] at (2.6, 0.45) {prefill};
\node[font=\scriptsize, gray, right] at (2.6, 0.12) {decode};

\foreach \x in {0,1,2} {
    \node[token=blue] at (-0.3+\x*0.42, 1.2) {};
}
\node[lbl, right] at (1.0, 1.2) {\scriptsize Class 1};

\foreach \x in {0,1,2,3} {
    \node[token=orange] at (-0.3+\x*0.42, 0.65) {};
}
\node[lbl, right] at (1.4, 0.65) {\scriptsize Class 2};

\node[token=blue] at (-0.3,-0.2) {};
\node[lbl, right] at (0.05,-0.2) {\scriptsize Class 1};

\node[token=orange] at (1.5,-0.2) {};
\node[lbl, right] at (1.85,-0.2) {\scriptsize Class 2};

\node[token=blue] at (-0.3,-0.7) {};
\node[token=orange] at (1.5,-0.7) {};

\draw[arr, blue!60] (q1p.east) -- (-1.0, 1.0);
\draw[arr, blue!60] (q1d.east) -- (-1.0, 0.0);
\draw[arr, orange!60] (q2p.east) -- (-1.0, 0.5);
\draw[arr, orange!60] (q2d.east) -- (-1.0, -0.5);

\node[font=\footnotesize] at (1.2,-1.5) {token budget: $\sum\text{tokens} \le b_{\max}$};

\end{tikzpicture}
\caption{Multi-class batch serving. Two classes of requests (e.g., ``Generate'' and ``Verify'') share the same LLM server. Each class maintains its own prefill (P) and decode (D) buffers. The GPU batches tokens from all classes together, subject to the token budget $b_{\max}$.}
\label{fig:multi-class-batch}
\end{figure}

In this section, we analyze the stability of such networks. We show that work-conserving scheduling remains throughput-optimal when a single LLM server handles multiple request classes (\cref{sec:single-multi-class}), when multiple servers are connected via DAG routing (\cref{sec:dag-network}), and when the network has fork-join synchronization (\cref{sec:fork-join}). However, when the server-level routing graph contains \emph{cycles}, work-conserving scheduling alone is no longer sufficient: inspired by the Rybko--Stolyar network~\citep{rybko1992ergodicity}, we construct an example where a work-conserving policy causes instability despite the load condition being satisfied (\cref{sec:RS-network}). Together, these results suggest that DAG routing structures are preferred in multi-server LLM deployments, as they guarantee throughput optimality under any work-conserving policy; when routing cycles are unavoidable, the scheduling policy must be designed with greater care.

\subsection{A single LLM server with multiple request classes}
\label{sec:single-multi-class}

Consider a single LLM server that handles $J$ classes of requests.  
Each class~$j$ has a prefill
phase and a decode phase with  mean prefill and decode token sizes $(m^j_p, m^j_d)$. All classes share the same LLM server and can be batched together (\cref{fig:multi-class-batch}).
 Upon completing service in class~$i$, a request is routed to class~$j$ with probability $p_{ij}$ and departs the system with probability $1-\sum_{j}p_{ij}$.
Let $P=(p_{ij})$ be the $J\times J$ transion matrix.
 We assume $I-P$ is invertible, which means that each request will eventually leave the system.

This model arises, e.g., from different stages of an agentic workflow routed to the same LLM server.
 Let 
 $\alpha^j$ denote the external arrival rate to class~$j$.
 The \emph{effective} arrival rate $\lambda^j$ to each class is determined by the traffic equations
\begin{align}\label{eq:traffic}
  \lambda^j = \alpha^j + \sum_{i=1}^{J}\lambda^i p_{ij}, \quad j=1,\ldots,J.
\end{align}
 The load condition generalizes naturally:
\begin{align}\label{eq:load-multi-class}
  \sum_{j=1}^{J}\lambda^j(m^j_p+m^j_d) < b_{\max}/t_{b_{\max}}.
\end{align}
For this  multiclass queueing model as well as other queueing models in the rest of this section, we need to define state $X(n)$ at time $n$, similar to the definition in (\ref{eq:state}) for the basic LLM queueing model. 
Define 
\begin{align}
  \label{eq:state-multiclass}
   X(n) =\Big(R(n), \{(P^j_i(n), D^j_i(n)) : i \in \mathcal{Q}^j_n\}, j=1, \ldots, J\Big).
\end{align}  
where superscript $j$ 
indicates a class $j$ component. Under a suitable distributional assumption on inter-request and token size distributions and a scheduling algorithm, $\{X(n), n\ge 0\}$ is a DTMC. (For other queueing models in this section, we will skip the explicit definition of $X(n)$.)

\begin{theorem}\label{thm:multi-class}
Under the load condition \cref{eq:load-multi-class}, any work-conserving $K$-FCFS scheduling algorithm stabilizes the system (i.e., the DTMC is positive recurrent).
\end{theorem}

\begin{proof}[Proof sketch]
The fluid model approach  developed in \cref{sec:basic} for the basic LLM queueing model extends easily to the multiclass setting. We will prove the fluid model in the current setting is stable.
Recall the  fluid model 
for the basic LLM queueing model is defined via \cref{eq:f1}-\cref{eq:f5} in \cref{sec:fm}. In the current multiclass setting, the fluid model is defined analogously via
\begin{align}
  &  Z^j_p(t)=Z^j_p(0)+ \alpha^j t + \sum_{i=1}^J p_{ij} F^i_d(t)- F^j_p(t), \label{eq:mf1} \\
  &  Z^j_d(t)= Z^j_d(0)+  F^j_p(t) - F^j_d(t),\label{eq:mf2} \\
  & B^j_f(0)=0, \qquad 0\le  B^j_f(t )-B^j_f(s) \le (t-s) b_{\max}/t_{b_{\max}}, \quad 0\le  s\le  t,
    \quad f\in\{p, d\}.\label{eq:mf3}\\
  &   F^j_f(t) = \frac{1}{ m^j_f }B^j_f(f), \quad f\in \{p, d\}, \label{eq:mf4} \\
  &   W^j_f(t) = m^j_f Z^j_f(t), \quad f\in \{p, d\}.   \label{eq:mf5}   
\end{align}
for each class $j$ and any time $t\in \R_+\equiv[0, \infty)$.
In addition, under work-conserving scheduling algorithm, the fluid model satisfies
\begin{align}
\sum_j(\dot B^j_p(t) + \dot B^j_d(t))=b_{\max}/t_{b_{\max}} \quad \text{whenever}
\quad \sum_j(Z^j_p(t)+Z^j_d(t))>0
\end{align}
for each regular point $t$.

From \cref{eq:mf1} and \cref{eq:mf2}, 
we have, in vector form, 
\begin{align*}
    &Z_p(t) = Z_p(0) + \alpha t + P' F_d(t) - F_p(t),\\
    &Z_p(t)+Z_d(t) = Z_p(0)+Z_d(0) + \alpha t - (I-P') F_d(t).
\end{align*}
These two equalities lead to 
\begin{align}\label{eq:mz3}
  &(I-P')^{-1} ( Z_p(t)+Z_d(t)) = (I-P')^{-1}(Z_p(0)+Z_d(0)) + \lambda t -  F_d(t), \\
  & Z_p(t)+P'(I-P')^{-1} ( Z_p(t)+Z_d(t)) = Z_p(0)+P'(I-P')^{-1} ( Z_p(0)+Z_d(0)) + \lambda t - F_p(t) \label{eq:mz4}
\end{align} 
where in obtaining \cref{eq:mz4}, we have used the fact $\alpha +P'\lambda=\lambda$.

Let $M_p={\rm diag}(m_p) $, $M_d={\rm diag}(m_d)$,
and $e$ be the row vector of ones.
Define Lyapunov function
\begin{align*}
   f(t)& =e M_d (I-P')^{-1}(Z_p(t)+Z_d(t)) + e M_p 
   \Big(Z_p(t)+P'(I-P')^{-1} ( Z_p(t)+Z_d(t))\Big).
\end{align*}
It is clear that $f(t)>0$ is equivalent to $\sum_j (Z^j_p(t)+Z^j_d(t))>0$.
It follows from \cref{eq:mz3}-\cref{eq:mz4} that
\begin{align*}
    f(t)=f(0)+t \sum_j \lambda_j (m^j_p+m^j_d) - \sum_j (B_p^j(t)+B_d^j(t)).
\end{align*}
 Therefore, 
at each regular point with $f(t)>0$,
\begin{align*}
  \dot f(t) & = \sum_{j=1}^{J}\alpha^j\,\mu^j - \sum_{j=1}^{J}(\dot B^j_p(t)+\dot B^j_d(t)) \\
  & = \sum_{j=1}^{J}\lambda^j(m^j_p+m^j_d) - b_{\max}/t_{b_{\max}} \\
  & = -\delta < 0,
\end{align*}
proving that $f(t)=0$ for $t\ge f(0)/\delta$. Therefore, we have proved  the fluid model is stable. Fluid model stability implies positive recurrence of the DTMC by the same argument as in \cref{sec:basic}.
\end{proof}

This result confirms that work-conserving scheduling remains throughput-optimal for a single server regardless of how many request classes it serves---the total workload argument applies as long as all classes share the same batch.

\subsection{Multi-server DAG networks}
\label{sec:dag-network}

We now consider multiple LLM servers whose routing relationships form a directed acyclic graph (DAG). There are $S$ servers, where server~$s$ handles a set~$\mathcal{J}_s$ of request classes (both prefill and decode for each class are processed on the same server). Upon completing class~$i$ at server~$\sigma(i)$, a request is routed to class~$j$ at server~$\sigma(j)$ with probability~$p_{ij}$, and departs with probability $1-\sum_j p_{ij}$. Routing across servers forms a DAG, namely, servers can be indexed $s=1,\ldots,S$ so that
for each $s'<s$, 
\begin{align}
    p_{i,j}=0 \quad \text{ for each }\quad i\in \mathcal{J}_s,\quad 
    j\in \mathcal{J}_{s'}.
\end{align}
In DAG, there are no directed cycles among distinct servers (though routing between classes on the \emph{same} server is permitted and handled by \cref{thm:multi-class}).

The effective arrival rates $(\lambda_j)$ satisfy the traffic equations analogous to \cref{eq:traffic}, now across all classes on all servers. The per-server load condition is
\begin{align}\label{eq:load-dag}
  \sum_{j\in\mathcal{J}_s}\lambda^j(m^j_p+m^j_d) < c_s, \quad s=1,\ldots,S,
\end{align}
where $c_s = b^s_{\max}/t^s_{b^s_{\max}}$ is the processing capacity of server~$s$.

\begin{theorem}\label{thm:dag}
Under the per-server load condition \cref{eq:load-dag} and work-conserving $K$-FCFS scheduling at each server, the DTMC describing the network is positive recurrent.
\end{theorem}

\begin{proof}[Proof sketch]
We prove the corresponding fluid model is stable, following Remark 8.16 of \cite{dai2020processing}.
Order the servers topologically: $s_1,s_2,\ldots,s_S$ so that inter-server routing only goes from earlier to later servers. We prove by induction that, in the fluid model, each server's Lyapunov function (from \cref{thm:multi-class}) drains to zero in finite time.

\emph{Base case.} Server~$s_1$ receives only external arrivals (no inter-server routing targets it). Applying the Lyapunov argument of \cref{thm:multi-class} to the classes $\mathcal{J}_{s_1}$, the per-server Lyapunov $f^{s_1}(t)$ drains to zero by time $T_1=f^{s_1}(0)/\delta_{s_1}$, where $\delta_{s_1} = c_{s_1} - \sum_{j\in\mathcal{J}_{s_1}}\lambda^j(m^j_p+m^j_d)>0$. For $t\ge T_1$, server~$s_1$ operates in balance with departure rates matching the effective arrival rates.

\emph{Inductive step.} Suppose all upstream servers $s_1,\ldots,s_{i-1}$ have drained by time $T_{i-1}=\max_{j<i}T_j$. For $t\ge T_{i-1}$, each upstream server departs at its nominal effective rate, so server~$s_i$ receives arrivals at the rate $\lambda^j$ for each class $j\in\mathcal{J}_{s_i}$. The same Lyapunov argument gives $\dot f^{s_i}(t)=-\delta_{s_i}<0$ whenever $f^{s_i}(t)>0$. During the transient phase $[0,T_{i-1}]$, upstream servers may produce departures faster than their steady-state rates as they drain initial backlogs, but in the fluid model ($|Z(0)|\le 1$) the total excess arrivals are bounded, so $f^{s_i}(T_{i-1})<\infty$. Therefore $f^{s_i}(t)=0$ for $t\ge T_i = T_{i-1}+f^{s_i}(T_{i-1})/\delta_{s_i}$.

All servers drain by time~$T_S$, proving fluid stability and hence positive recurrence.
\end{proof}

This result shows that work-conserving scheduling remains throughput-optimal across any DAG of LLM servers, as long as each server's load condition is satisfied. The DAG structure ensures no feedback loops between servers, allowing a cascading stability argument.

\subsection{Fork-join networks}
\label{sec:fork-join}

The DAG result (\cref{thm:dag}) uses probabilistic routing, where each completed request routes to one downstream class. Fork-join networks introduce two features not captured by this model: (i)~\emph{forking}, where a request spawns multiple sub-tasks processed in parallel at different servers, and (ii)~\emph{join synchronization}, where all sub-tasks must complete before the request proceeds. This pattern arises naturally when one LLM call triggers parallel tool queries, searches, or sub-agent invocations (\cref{fig:agentic-workflows}(b)).

Consider $k+1$ LLM servers indexed by $q_1,q_2,\ldots,q_{k+1}$. Requests arrive at server~$q_1$ at rate~$\lambda$. Upon completing processing at~$q_1$, each request forks into $k-1$ sub-tasks that are sent to servers $q_2,\ldots,q_k$ (processed in parallel). Once \emph{all} $k-1$ sub-tasks complete, the results merge and the request is processed by server~$q_{k+1}$ before departing. Let $(m^s_p, m^s_d)$ denote the mean prefill and decode token sizes at server~$q_s$, and let $c_s = b^s_{\max}/t^s_{b^s_{\max}}$ denote the processing capacity of server~$q_s$.

\begin{theorem}\label{thm:fork-join}
Assume the load condition holds at each server:
\begin{align}
  \label{eq:fj}
    \lambda(m^s_p+m^s_d) < c_s \quad  \text{ for each }
s\in \{1, \ldots, k+1\}.
\end{align}
Under any $K$-FCFS work-conserving scheduling algorithm at each server, the DTMC describing the system is positive recurrent.
\end{theorem}

The key challenge is the \emph{join} synchronization: a request cannot proceed to~$q_{k+1}$ until all $k-1$ sub-tasks at $q_2,\ldots,q_k$ complete. The $K$-FCFS condition provides a way forward through \emph{bounded overtaking} (\cref{lem:bounded-overtaking} in \cref{sec:dynamics}): at each server in $q_1, q_2,\ldots,q_{k+1}$, the completion order is within a constant~$K'$ of the arrival order. Since all sub-tasks of a given request arrive at $q_2,\ldots,q_k$ in the same order (forked from~$q_1$), the completion orders across fork servers differ by at most~$2K'$. This bounds the total number of 
requests in "join buffers" by $2K'(k-1)$, which vanishes under fluid scaling. Once the join buffers vanish in the fluid limit, the cascading argument of \cref{thm:dag} applies: $q_1$ drains first, then $q_2,\ldots,q_k$ drain, then the join releases at rate~$\lambda$, and finally $q_{k+1}$ drains. The full proof is given in \cref{sec:proof-fork-join}.

\begin{remark}
\Cref{thm:fork-join} readily extends to a multi-server DAG network (\cref{sec:dag-network}) in which individual classes may have fork-join internal structure. From each server's perspective, fork sub-tasks are simply additional classes that arrive, are processed, and depart; the per-server Lyapunov argument of \cref{thm:multi-class} applies, and the cascading argument of \cref{thm:dag} carries through with the join buffers controlled by bounded overtaking.
\end{remark}

\paragraph{Related work on fork-join networks.}
The stability of fork-join queueing networks under FIFO service was established by \citet{konstantopoulos1989fork} and \citet{baccelli1989fork}. Both works proved that the natural load condition suffices for stability when each server operates under strict FIFO; \citet{baccelli1989fork} further derived stochastic ordering bounds on system response times. More recently, \citet{ozkan2016fork} studied optimal scheduling control at a shared server in a fork-join network under heavy traffic, and \citet{gao2025fork} characterized stability regions for fork-join systems with redundancy and heterogeneous servers under FCFS. \Cref{thm:fork-join} extends the stability guarantee to the broader class of $K$-FCFS policies, which permit bounded reordering and are natural in the LLM batch-service setting where multiple requests are processed simultaneously. The key enabling technique is the \emph{bounded overtaking} property (\cref{lem:bounded-overtaking} in \cref{sec:dynamics}), which controls the join buffer sizes under $K$-FCFS; this extension from FIFO to $K$-FCFS may be of independent interest for fork-join networks beyond the LLM context.

\subsection{Work-conserving algorithms can fail in networks: a Rybko--Stolyar example}
\label{sec:RS-network}

While work-conserving scheduling is throughput-optimal for a single server (\cref{thm:main}), DAG networks (\cref{thm:dag}), and fork-join networks (\cref{thm:fork-join}), it is \emph{not} sufficient when the server-level routing graph contains cycles. We demonstrate this through a Rybko--Stolyar (RS) type network~\citep{rybko1992ergodicity}, adapted to the LLM batch-service setting.

Consider two LLM servers indexed by $j \in \{1, 2\}$ (see \cref{fig:agent-illu}). Two types of requests flow through the network:
\begin{itemize}
    \item \textbf{Type A (product search)}: first a short task at server~1, then a long task at server~2.
    \item \textbf{Type B (customer complaint)}: first a short task at server~2, then a long task at server~1.
\end{itemize}
Such workflows arise naturally in multi-agent systems. For instance, in an e-commerce setting, server~1 may run a model specialized for product retrieval (query parsing, RAG, verification) while server~2 runs a different model specialized for communication (emotional tone, compliance). A product search (type~A) starts at server~1 and then routes to server~2; a customer complaint (type~B) flows in the reverse direction. Because the two request types route in opposite directions, the resulting server-level routing graph contains a \emph{cycle}.

\begin{figure}[t]
    \centering
\begin{tikzpicture}[
    >=Stealth,
    arr/.style={->, thick},
    every node/.style={font=\small}
]

\draw[thick, rounded corners=4pt, fill=blue!8] (-1.8,-1.35) rectangle (1.8,1.35);
\draw[thick, rounded corners=4pt, fill=blue!8] (4.7,-1.35) rectangle (8.3,1.35);

\node[font=\small\bfseries] at (0, 0) {LLM Server 1};
\node[font=\small\bfseries] at (6.5, 0) {LLM Server 2};

\draw[fill=orange!25, thick] (-1.45, 0.5) rectangle (-1.05, 0.9);

\draw[fill=green!25, thick] (0.55, -0.9) rectangle (0.95, -0.5);
\draw[fill=green!25, thick] (1.05, -0.9) rectangle (1.45, -0.5);

\draw[fill=green!25, thick] (5.05, 0.5) rectangle (5.45, 0.9);
\draw[fill=green!25, thick] (5.55, 0.5) rectangle (5.95, 0.9);

\draw[fill=orange!25, thick] (7.55, -0.9) rectangle (7.95, -0.5);

\draw[arr] (-3.6, 0.7) -- (-1.8, 0.7);
\node[font=\small, left] at (-3.6, 0.7) {Type A\, (s,s)};

\draw[arr] (9.3, -0.7) -- (8.3, -0.7);
\node[font=\small, right] at (9.3, -0.7) {Type B\, (s,s)};

\draw[arr] (1.8, 0.7) -- (4.7, 0.7);
\node[font=\small, above] at (3.25, 0.7) {(l,s)};

\draw[arr] (4.7, -0.7) -- (1.8, -0.7);
\node[font=\small, below] at (3.25, -0.7) {(l,s)};

\end{tikzpicture}
    \caption{Rybko--Stolyar type network of two LLM servers. (s,s) denotes short prefill and short decode; (l,s) denotes long prefill and short decode.}
    \label{fig:agent-illu}
\end{figure}

We simulate this network with the following parameters. Requests arrive as a Poisson process with equal rates for types~A and B, at total load $\rho = 0.9$. Each server has token budget $b_{\max} = 768$ and constant batch processing time $t_b = 1$. Token sizes are i.i.d.\ with means
\begin{align}
  \label{eq:agent1}
  (m^{(A,1)}_{p}, m^{(A,1)}_d)=(32,32), \quad (m^{(B,1)}_{p}, m^{(B,1)}_d)=(512,32)
\end{align}
for server~1, and
\begin{align}
  \label{eq:agent2}
  (m^{(A,2)}_{p}, m^{(A,2)}_d)=(512,32), \quad (m^{(B,2)}_{p}, m^{(B,2)}_d)=(32,32)
\end{align}
for server~2. Both servers use work-conserving scheduling with mixed batching. Server~1 gives non-preemptive priority to type~B sub-tasks, and server~2 gives non-preemptive priority to type~A sub-tasks. Within each type, decode-priority FCFS is used.

\Cref{fig:agent-comparison} (top row) shows that despite $\rho < 1$, the total number of tokens grows without bound. This demonstrates that \emph{not all work-conserving policies are stable for networks of LLM agents}: the choice of priority order matters.

Reversing the priority assignment---server~1 prioritizes type~A, server~2 prioritizes type~B---restores stability (\cref{fig:agent-comparison}, bottom row). One can prove positive recurrence of this configuration using the fluid limit approach of \cref{sec:basic}. This instability phenomenon is the LLM batch-service analog of the classical Rybko--Stolyar example~\citep{rybko1992ergodicity,dai2020processing}.

\paragraph{Design principle: avoid routing cycles across servers.}
The root cause of instability above is the \emph{cycle} in the server-level routing graph: type~A flows from server~1 to server~2, while type~B flows from server~2 to server~1. By contrast, \cref{thm:dag,thm:fork-join} guarantee that any work-conserving policy is throughput-optimal whenever the server-level routing graph is a DAG (with or without fork-join synchronization). This suggests a practical design principle for multi-server LLM deployments: when assigning agentic workloads to servers, one should strive to arrange the server-level routing as a DAG, avoiding cycles across servers whenever possible. When cycles are unavoidable, the scheduling policy must be chosen with care---not all work-conserving policies remain stable. In the conventional multiclass queueing network setting, dynamic scheduling algorithms such as back-pressure~\citep{tassiulas1990stability} and proportionally fair resource allocation~\citep{kelly1997charging,kelly1998rate}, including HLPPS~\citep{bramson1996convergence}, have been proved to be throughput optimal. See, also Chapters 9 and 10 of \cite{dai2020processing}.  These algorithms can be adapted to the setting of 
networks of LLM servers.  Understanding which algorithms can robustly balance the  throughput and latency trade-off  for general agentic workload  is an important direction for future work.

\begin{figure}[ht]
    \begin{subfigure}{0.49\linewidth}
        \centering
        \includegraphics[width=\linewidth]{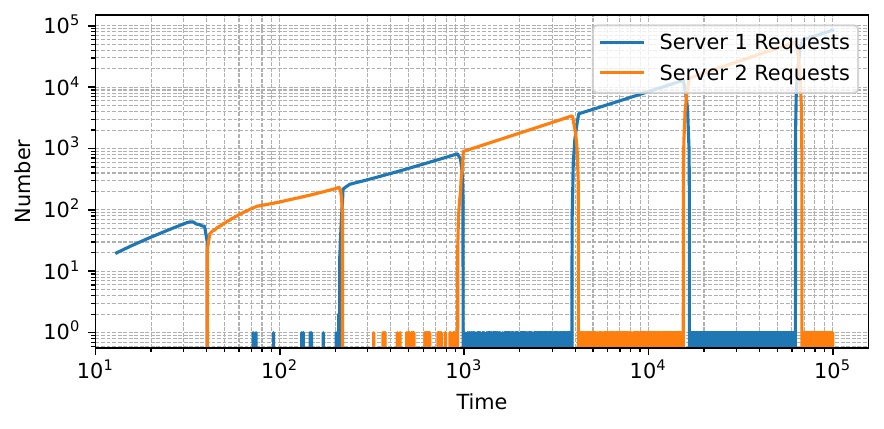}
    \end{subfigure}
    \hfill
    \begin{subfigure}{0.49\linewidth}
        \centering
        \includegraphics[width=\linewidth]{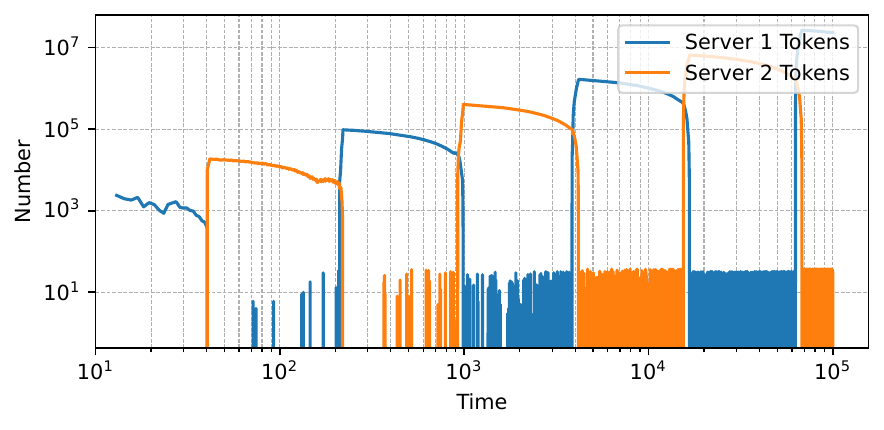}
    \end{subfigure}\vspace{0.5em}
    \centering
    \begin{subfigure}{0.49\linewidth}
        \centering
        \includegraphics[width=\linewidth]{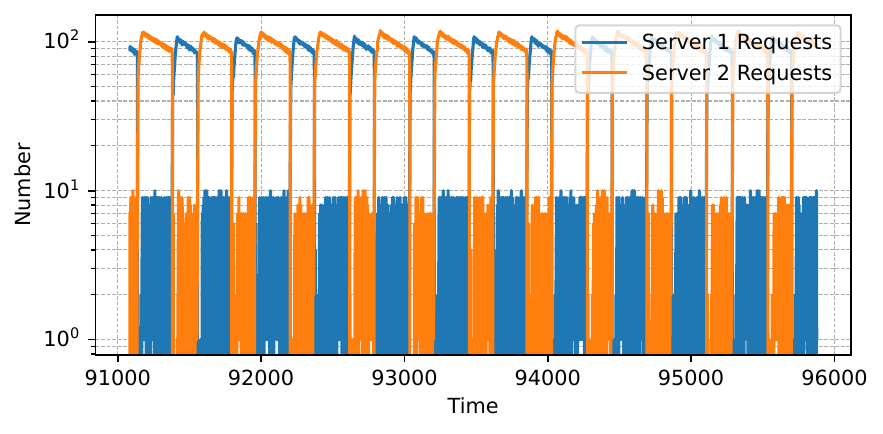}
    \end{subfigure}
    \hfill
    \begin{subfigure}{0.49\linewidth}
        \centering
        \includegraphics[width=\linewidth]{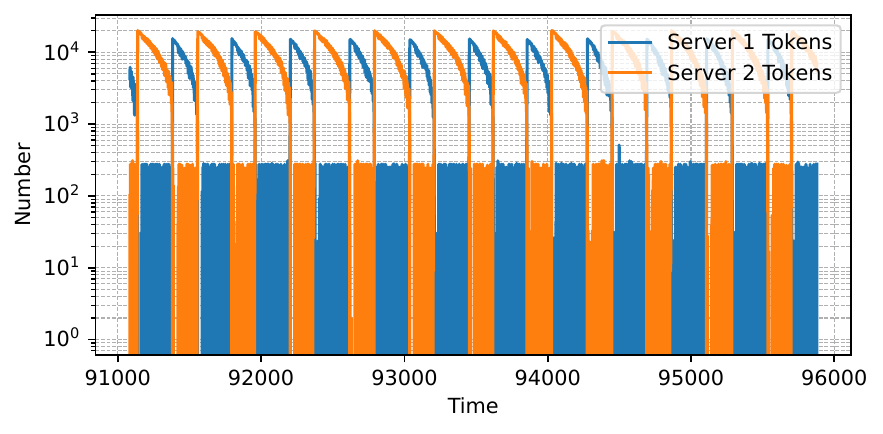}
    \end{subfigure}
    \caption{Rybko--Stolyar network at $\rho = 0.9$. Top row: under the destabilizing priority assignment, total requests (left) and tokens (right) grow without bound. Bottom row: reversing the priority restores stability.}
    \label{fig:agent-comparison}
\end{figure}

%% file: sections/sec6_batch_size_constraint.tex
\section{Extension to Maximal Batch Size Constraint}\label{sec:extension-to-active-request-constraint}

Throughout this paper, the only constraint on a batch has been the token budget $b \le b_{\max}$. This is the primary bottleneck in practice: the batch processing time $t_b$ is determined by the total number of tokens, and GPU compute is the dominant limiting resource. In this regime, \cref{thm:main} provides a clean scalar threshold for stability, and work-conserving scheduling is throughput-optimal.

In practice, however, a maximal batch size constraint $k \le k_{\max}$ is often imposed to limit GPU memory usage and control tail latency. When $k_{\max}$ is large relative to the typical number of tokens per request, this constraint is rarely binding and can be safely ignored. However, when $k_{\max}$ is small (e.g., for long-context models or memory-constrained deployments), the batch size constraint can fundamentally alter the stability region.

In this section, we incorporate the batch size constraint $k \le k_{\max}$ and show how it complicates the analysis. The stability region is no longer a single scalar threshold but a convex region in the $(\lambda_p, \lambda_d)$ plane, where $\lambda_p := \lambda m_p$ and $\lambda_d := \lambda m_d$ are the prefill and decode token arrival rates, respectively. We believe that work-conserving scheduling remains a sound principle to advocate in general, but we demonstrate that under extreme configurations, certain work-conserving algorithms can fail while some non-work-conserving algorithms succeed. We further experimentally verify this phenomenon and leave a thorough study to future work.

\subsection{Stability region under the batch size constraint}

We now consider both the token budget and the batch size constraint. Each batch is limited to at most $k_{\max}$ requests and at most $b_{\max}$ tokens. For notational convenience in this section, let $k_p$ and $k_d$ denote the number of prefill and decode requests in a batch, and let $x$ and $y$ denote the total prefill and decode tokens, respectively. The constraints are:
\begin{align}
    k_p + k_d &= k \leq k_{\max}, \label{eq:memory_constraint} \\
    x + y &= b \leq b_{\max}. \label{eq:batch_constraint}
\end{align}

To characterize the stability region, we decompose the arrival workload into prefill and decode components $\lambda_p$ and $\lambda_d$. In stationarity, positive recurrence requires the expected processing rate of each token type to match or exceed its arrival rate.

\begin{condition}[Necessary condition for stability]
\label{stable_condition}
Let $u_{x,y}$ denote the long-run fraction of time spent processing a batch with $x$ prefill tokens and $y$ decode tokens. The system is positive recurrent only if there exists a probability vector $u = (u_{x,y})_{x,y \ge 0}$ with $\sum u_{x,y} \le 1$ such that
\begin{align}
    \lambda_p &\le \sum_{\substack{x+y \le b_{\max} \\ y \le k_{\max}-1}} \frac{x}{t_{x+y}}\, u_{x,y}, \label{eq:stable_condition_lambda_p} \\
    \lambda_d &\le \sum_{\substack{x+y \le b_{\max} \\ y \le k_{\max}-1}} \frac{y}{t_{x+y}}\, u_{x,y} + \frac{k_{\max}}{t_{k_{\max}}}\, u_{0,k_{\max}}. \label{eq:stable_condition_lambda_d}
\end{align}
\end{condition}
The right-hand side of \cref{eq:stable_condition_lambda_p} is the long-run average rate at which prefill tokens are processed: a batch with $x$ prefill tokens and processing time $t_{x+y}$ contributes $x/t_{x+y}$ prefill tokens per unit time, weighted by the fraction $u_{x,y}$. For stability, this must be at least $\lambda_p$. Similarly, \cref{eq:stable_condition_lambda_d} requires the decode processing rate to match $\lambda_d$; the additional term $u_{0,k_{\max}}$ accounts for decode-only batches where all $k_{\max}$ slots are occupied by decoding requests, each producing one token per step , with no prefill processing.

While the necessary condition \cref{eq:stable_condition_lambda_p}--\cref{eq:stable_condition_lambda_d} involves an optimization over infinitely many variables $u_{x,y}$, it admits a surprisingly clean characterization under the piecewise constant processing time model (\cref{eq:batch-processing-time}): the stability region reduces to the convex hull of just four vertices.

\begin{theorem} \label{convex_hull}
Assume batch processing times follow \cref{eq:batch-processing-time} with $a > 0$, $b_0 \mid b_{\max}$, and $k_{\max} \le b_0$. The pair $(\lambda_p, \lambda_d) \in \mathbb{R}^2_{\ge 0}$ satisfies the necessary condition \cref{eq:stable_condition_lambda_p}--\cref{eq:stable_condition_lambda_d} if and only if $(\lambda_p, \lambda_d)$ lies in the convex hull of the following four vertices:
\begin{align*}
    A &= \Bigl(\frac{b_{\max}}{t_{b_{\max}}},\, 0\Bigr), &
    B &= \Bigl(\frac{b_{\max}-k_{\max}+1}{t_{b_{\max}}},\, \frac{k_{\max}-1}{t_{b_{\max}}}\Bigr), \\
    C &= \Bigl(\frac{b_0-k_{\max}+1}{t_{b_0}},\, \frac{k_{\max}-1}{t_{b_0}}\Bigr), &
    D &= \Bigl(0,\, \frac{k_{\max}}{t_{k_{\max}}}\Bigr).
\end{align*}
Here, the inequality $(\lambda_p, \lambda_d) \le (a, b)$ is interpreted componentwise.
\end{theorem}

\Cref{fig:convexhull} illustrates the stability region. Each vertex corresponds to a distinct operating regime:
\begin{enumerate}
    \item[\textbf{(A)}] The batch consists entirely of prefill tokens, filling the token budget $b_{\max}$.
    \item[\textbf{(B)}] The batch mixes $k_{\max}-1$ decode tokens with $b_{\max}-k_{\max}+1$ prefill tokens at the maximal token budget.
    \item[\textbf{(C)}] The batch mixes $k_{\max}-1$ decode tokens with $b_0 - k_{\max} + 1$ prefill tokens at the smaller token load $b_0$.
    \item[\textbf{(D)}] The batch consists solely of $k_{\max}$ decode tokens, constrained by the batch size limit.
\end{enumerate}

Vertex $C$ reveals a counterintuitive phenomenon: it is sometimes optimal to form batches with token load $b_0$ rather than $b_{\max}$. This occurs because using the full token budget $b_{\max}$ with many decode requests forces the batch processing time to a higher step of $t_b$, reducing the per-token processing rate. By operating at token load $b_0$, the system achieves a more favorable rate for certain arrival configurations.

\begin{figure}[t]
\centering
\begin{tikzpicture}[scale=2.6, every node/.style={font=\footnotesize}]
  \draw[->] (0,0) -- (2.1,0) node[below] {$\lambda_p$};
  \draw[->] (0,0) -- (0,2.1) node[left] {$\lambda_d$};

  \coordinate (A) at (1.5,0);
  \coordinate (B) at (1.4,0.2);
  \coordinate (C) at (0.3,1.43);
  \coordinate (D) at (0,1.5);
  \coordinate (O) at (0,0);

  \draw[thick, fill=blue!10] (O.center) -- (A) -- (B) -- (C) -- (D) -- cycle;

  \foreach \p in {A,B,C,D} \fill (\p) circle (0.6pt);

  \node[below, font=\scriptsize] at (A) {$A\;\bigl(\tfrac{b_{\max}}{t_{b_{\max}}},\, 0\bigr)$};
  \node[above right=1pt, font=\scriptsize] at (B) {$B\;\bigl(\tfrac{b_{\max}-k_{\max}+1}{t_{b_{\max}}},\, \tfrac{k_{\max}-1}{t_{b_{\max}}}\bigr)$};
  \node[above right=2pt, font=\scriptsize] at (C) {$C\;\bigl(\tfrac{b_0-k_{\max}+1}{t_{b_0}},\, \tfrac{k_{\max}-1}{t_{b_0}}\bigr)$};
  \node[left=4pt, font=\scriptsize] at (D) {$D\;\bigl(0,\, \tfrac{k_{\max}}{t_{k_{\max}}}\bigr)$};

\end{tikzpicture}
\caption{Stability region under both the token budget $b_{\max}$ and batch size constraint $k_{\max}$. The achievable $(\lambda_p, \lambda_d)$ pairs lie within the convex hull $O$--$A$--$B$--$C$--$D$.}
\label{fig:convexhull}
\end{figure}

\subsection{Work-conserving algorithms can fail}

The existence of vertex $C$ has a striking consequence: work-conserving algorithms---which always fill batches to the full token budget $b_{\max}$---may fail to stabilize the system at arrival rates near the segment $\overline{CD}$.

We demonstrate this with a concrete experiment. We use CodeLlama-34B with TP1, whose batch processing time is modeled by $t_b = 11.28 + 35.47 \times \lceil b/128 \rceil$ (see \cref{fig:batch-processing-time-linear-form}), giving $b_0 = 128$. We set $k_{\max} = 100$ and choose arrival parameters corresponding to vertex $C$: $(\lambda_p, \lambda_d) = (29/46.75,\, 99/46.75)$, with a fixed inter-arrival time of 467.5\,ms, 290 prefill tokens, and 990 decode tokens per request.

\Cref{fig:scheduler_comparison} shows the results. Under a token budget of $b_{\max} = 128$, Sarathi-Serve is the only algorithm that maintains stability. It constructs batches of exactly 29 prefill tokens and 99 decode tokens when there are 99 outstanding requests, precisely matching the operating point $C$. In contrast:
\begin{itemize}
    \item FasterTransformer and vanilla vLLM fail because they do not support mixed batching, preventing them from balancing prefill and decode processing rates.
    \item Orca, despite supporting mixed batching, is prefill-prioritized: when a new request arrives, it constructs consecutive batches dominated by prefill tokens, causing the decode processing rate to fall below vertex $C$.
\end{itemize}
\Cref{fig:sarathi_comparison} further shows that even Sarathi-Serve fails when the token budget is increased to $b_{\max} = 1024$. In this case, the algorithm fills batches to 1024 tokens, which accelerates prefill processing but causes the GPU to accumulate $k_{\max} = 100$ outstanding decode requests. Subsequent batches then consist entirely of 100 decode tokens at an inefficient rate, disrupting the balance.

\begin{figure}[ht]
  \centering
  \begin{subfigure}[t]{0.49\textwidth}
    \centering
    \includegraphics[width=\linewidth,height=0.60\textwidth]%
      {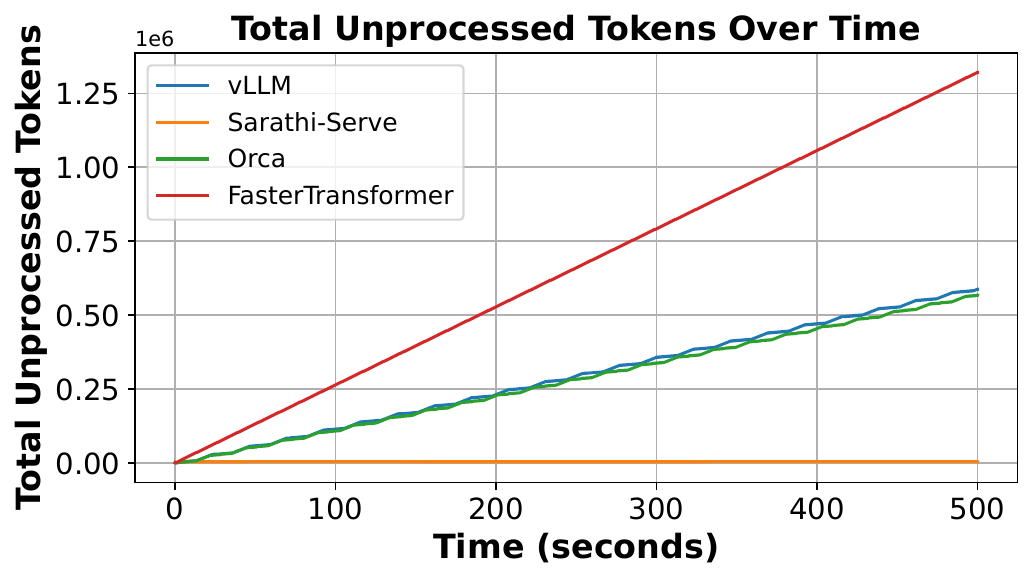}
    \caption{All schedulers with $b_{\max} = 128$.}\label{fig:scheduler_comparison}
  \end{subfigure}\hfill
  \begin{subfigure}[t]{0.49\textwidth}
    \centering
    \includegraphics[width=\linewidth,height=0.60\textwidth]%
      {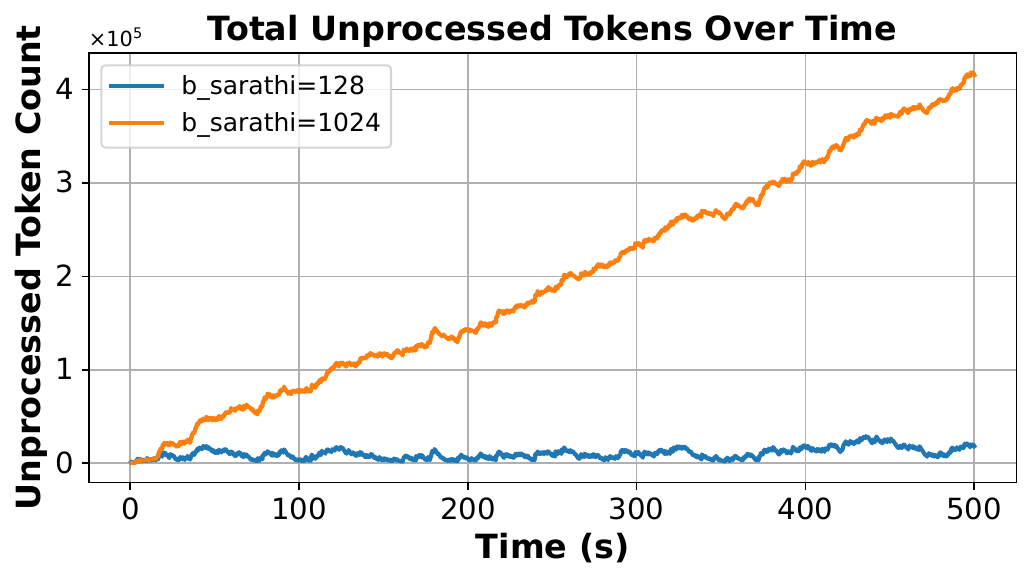}
    \caption{Sarathi-Serve with $b_{\max} \in \{128, 1024\}$.}\label{fig:sarathi_comparison}
  \end{subfigure}
  \caption{Unprocessed tokens over time at the operating point $C$. (a)~Only Sarathi-128 maintains stability. (b)~Sarathi-1024 fails despite being work-conserving, because filling to $b_{\max} = 1024$ disrupts the prefill--decode balance.}
  \label{fig:kmax_experiments}
\end{figure}

\subsection{Discussion}

The batch size constraint $k_{\max}$ introduces a fundamental asymmetry between prefill and decode tokens that the token budget alone does not capture. Because each decode request contributes exactly one token per batch but occupies one slot in the batch size limit, decode-heavy workloads are disproportionately affected by $k_{\max}$.

The stability region in \cref{convex_hull} provides a necessary condition for positive recurrence. An important open question is whether every point in this convex hull is achievable by some scheduling algorithm, particularly when the distributions of prefill and decode lengths vary. Designing an adaptive algorithm that guarantees stability for all arrival configurations within the convex hull---or proving that no such algorithm exists---remains an interesting direction for future work.

%% file: sections/sec7_future_directions.tex
\section{Future Directions and Conclusion}

We have developed a queueing-theoretic framework for LLM inference scheduling. Our main contributions are: (i)~a complete characterization of the stability region for a single LLM server under the token budget constraint, showing that work-conserving $K$-FCFS algorithms are throughput-optimal (\cref{thm:main}); (ii)~an extension to parallel servers (\cref{pro:parallel}) and to a single server with multiple request classes (\cref{thm:multi-class}); (iii)~stability results for multi-server DAG networks (\cref{thm:dag}) and fork-join networks (\cref{thm:fork-join}) under $K$-FCFS; (iv)~a Rybko--Stolyar example showing that work-conserving scheduling can fail in certain network topologies; and (v)~an analysis of the batch size constraint, revealing a richer stability region where work-conserving algorithms may not suffice. We conclude with preliminary results on latency and several open directions.

\textbf{Latency optimization.}
This paper focuses on throughput and stability. However, in practice, the choice of scheduling algorithm also affects latency, and the best policy depends on the workload and latency metric of interest. As a preliminary investigation, \cref{tab:latency-code} reports end-to-end (E2E) latency, time to first token (TTFT), and time between tokens (TBT) for the four scheduling algorithms on CodeLlama-34B under production coding traces at 1.2 queries per second (QPS). Sarathi-Serve achieves the best E2E and prefill latency across all percentiles, while FasterTransformer has the lowest TBT due to its decode-only batching---highlighting a trade-off between prefill responsiveness and decode smoothness.

\Cref{fig:ablation-study-conversation-trace} further shows how the token budget within Sarathi-Serve affects latency under conversation traces. A moderate budget (e.g., 512 tokens) yields the lowest median E2E latency, a larger budget (e.g., 1024 tokens) improves prefill latency, and an excessively small budget (e.g., 128 tokens) incurs significant overhead from repeated chunked prefill iterations. The key takeaway is that the optimal token budget depends on the distribution of prefill lengths and service level objectives (SLOs). A comprehensive study of latency optimization is an important direction for future work.

\begin{table}[!htp]
    \centering
\caption{Latency comparison of the four scheduling algorithms on CodeLlama-34B (single A100, 1.2 QPS, production coding traces).}\label{tab:latency-code}
\scriptsize
\begin{tabular}{lrrrrr}\toprule
E2E (ms) &Median &P90 &P95 &P99 \\
FasterTransformer &45.38 &116.81 &123.98 &139.65 \\
Orca &7.1 &16.13 &18.6 &30.59 \\
vLLM &6.22 &14.62 &20.74 &40.75 \\
Sarathi-Serve &\textbf{3.78} &\textbf{8.8} &\textbf{12.29} &\textbf{22.36} \\ \bottomrule
& & & & \\
TTFT (ms) &Median &P90 &P95 &P99 \\
FasterTransformer &44.61 &114.53 &123.63 &138.87 \\
Orca &4.55 &13.99 &14.83 &15.9 \\
vLLM &2.52 &7.07 &8.58 &10.09 \\
Sarathi-Serve &\textbf{2.02} &\textbf{3.87} &\textbf{5.81} &\textbf{7.47} \\ \bottomrule
& & & & \\
TBT (ms) &Median &P90 &P95 &P99 \\
FasterTransformer &\textbf{0.04} &\textbf{0.04} &\textbf{0.04} &\textbf{0.05} \\
Orca &0.08 &0.19 &0.24 &0.38 \\
vLLM &0.17 &0.43 &0.47 &0.65 \\
Sarathi-Serve &0.13 &0.17 &0.17 &0.17 \\
\bottomrule
\end{tabular}
\end{table}

\begin{figure}
    \centering
    \begin{subfigure}{0.49\linewidth}
        \centering
        \includegraphics[width=\linewidth]{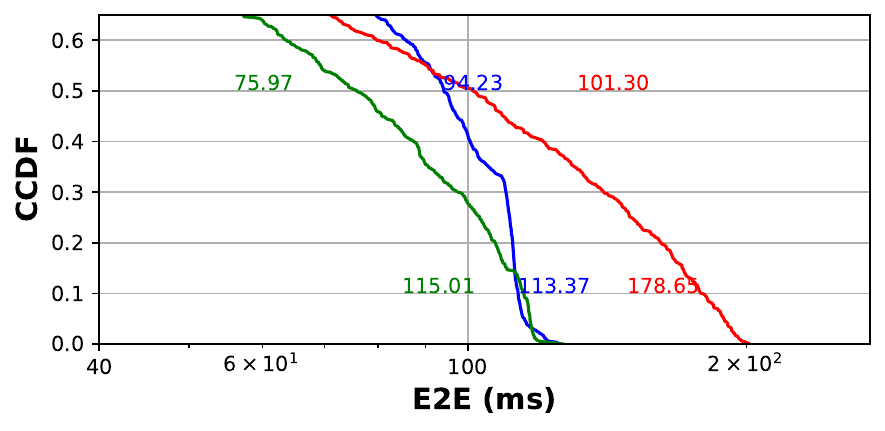}
    \end{subfigure}
    \hfill
    \begin{subfigure}{0.49\linewidth}
        \centering
        \includegraphics[width=\linewidth]{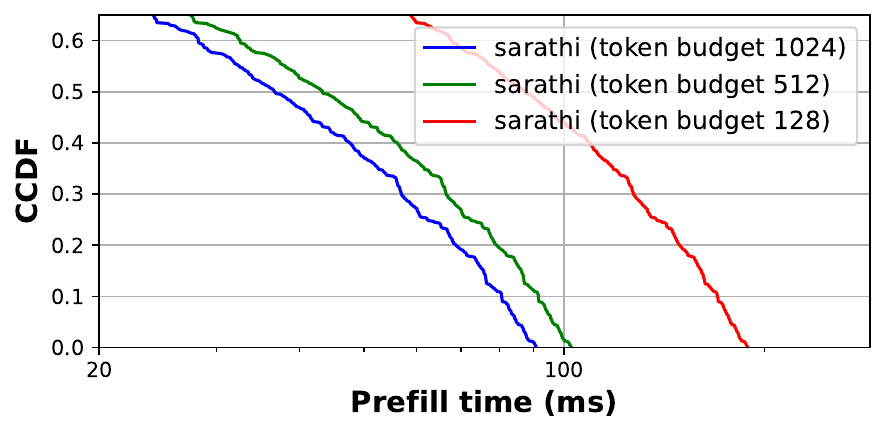}
    \end{subfigure} 
    \caption{E2E latency (left) and TTFT (right) for different token budgets of Sarathi-Serve on CodeLlama-34B under production conversation traces.}
    \label{fig:ablation-study-conversation-trace}
\end{figure}

\textbf{Other directions.}
Beyond latency, several important directions remain.
\begin{itemize}
    \item \textbf{Full treatment of the batch size constraint.} The stability region under the joint token budget and batch size constraint (\cref{sec:extension-to-active-request-constraint}) is only partially characterized. Designing an adaptive algorithm that achieves throughput optimality for all arrival configurations within the convex hull remains open.
    \item \textbf{Richer batch processing time models.} Our analysis uses the piecewise constant model $t_b = c + a\lceil b/b_0\rceil$, which captures the dominant linear-layer cost but omits attention-layer effects. Extending the stability analysis to more expressive models---for example, incorporating the total number of KV cache tokens as in \cite{ao2025optimizing}, or the fully analytical batch processing time characterization developed in \cite{bari2025optimal}---is an important direction. In particular, understanding how these richer models interact with multi-server AI-agent topologies (\cref{sec:AI-agent}) may reveal new scheduling trade-offs not captured by our current framework.
    \item \textbf{Tail-latency objectives.} Optimal scheduling under tail-latency measures (e.g., tail TBT or TTFT) under various load regimes warrants further study~\citep{yu2024strongly,zhang2025tail}.
    \item \textbf{System-level optimization.} Joint optimization of model autoscaling, resource allocation across multiple models, KV cache policies, and load balancing may yield substantial improvements, particularly under bursty or skewed workloads. Scheduling under multi-tenancy---where best-effort and latency-critical requests share the same servers---presents additional challenges.
\end{itemize}
We hope this work encourages collaborative efforts between queueing theorists and system practitioners to develop models that accurately capture LLM inference dynamics and inform the design of next-generation scheduling algorithms.

%% file: sections/appendix_a_batch_processing_time.tex
\section{Batch Processing Time Beyond the Linear Regime} \label{sec:discussions}

In the main text, \cref{eq:batch-processing-time} models the batch processing time as a function of the token load $b$ alone. This is justified when the linear-layer (i.e., feed-forward) computation dominates, which holds for moderate input contexts, output lengths, and batch sizes. However, for longer contexts or larger batches, additional computational components become non-negligible.

Recall from Section~\ref{sec:iteration-batch} that a batch is represented as $\{(x^i_{l_i}, \dotsc, x^i_{r_i}) : i \in \mathcal{B}\}$, with token load $b = \sum_{i \in \mathcal{B}} (r_i - l_i + 1)$. For each request $i$ in the batch, $(r_i - l_i + 1)$ new tokens are processed, and each attends to all tokens up to position $r_i$ (i.e., $l_i - 1$ previously processed tokens plus the new chunk). Note that for decode requests, the feasibility constraints in Section~\ref{sec:model-setup} enforce $r_i = l_i$.

A more complete model of the batch processing time $t_b$ takes the form
\begin{align}
    t_b \;=\; c \;+\; \underbrace{a_1 \sum_{i \in \mathcal{B}} (r_i - l_i + 1)}_{\text{linear layer}} \;+\; \underbrace{a_2 \sum_{i \in \mathcal{B}} (r_i - l_i + 1)^2}_{\text{self-attention}} \;+\; \underbrace{a_3 \sum_{i \in \mathcal{B}} (r_i - l_i + 1)(l_i - 1)}_{\text{cross-attention}},
\end{align}
where $c, a_1, a_2, a_3 \geq 0$ are hardware- and model-dependent constants. The self-attention term accounts for attention among the $(r_i - l_i + 1)$ new tokens within each chunk, while the cross-attention term accounts for each new token attending to the $l_i - 1$ previously processed tokens. Note that for decode requests ($r_i = l_i$), the self-attention cost reduces to $a_2$ and the cross-attention cost becomes $a_3(l_i - 1)$.

In typical LLM architectures, $a_1 \gg a_2, a_3$, because the linear-layer (feed-forward) computation involves large matrix multiplications that dominate per-token cost, while the attention computation per token pair is comparatively cheap. This separation of scales gives rise to three asymptotic regimes:
\begin{enumerate}
    \item \textbf{Moderate input, batch size, and output lengths:} The linear-layer term dominates:
    \[
        t_b \;\approx\; c + a_1 \sum_{i \in \mathcal{B}} (r_i - l_i + 1).
    \]
    This is the regime where \cref{eq:batch-processing-time} provides an accurate approximation, as the iteration time depends primarily on the total token load $b$. This regime covers many realistic serving traces (e.g., ShareGPT~\citep{sharegpt}).
    \item \textbf{Long output or large batch size:} The cross-attention term dominates via the decode requests (where $r_i = l_i$):
    \[
        t_b \;\approx\; c + a_3 \sum_{i \in \mathcal{B}} (l_i - 1).
    \]
    \item \textbf{Long input context:} The self-attention and cross-attention terms dominate via the prefill requests:
    \[
        t_b \;\approx\; c + \sum_{i \in \mathcal{B}} \bigl[a_2\,(r_i - l_i + 1)^2 + a_3\,(r_i - l_i + 1)(l_i - 1)\bigr].
    \]
\end{enumerate}

Our theoretical analysis focuses on the first regime. Extending the throughput-optimality results to the latter two regimes—where the batch processing time depends on the \emph{composition} of the batch, not just its total token load—is an important direction for future work. For a more detailed explanation for the modeling above, we refer the reader to \citet{bari2025optimal}.




%% file: sections/appendix_b_fluid_model_proof.tex
\section{Fluid Model Proof of \cref{thm:main} Part~(b)}
\label{sec:basic}
In this appendix, we prove Part~(b) (stability) of \cref{thm:main} using the fluid limit technique~\citep{dai1995positive,dai2020processing}. 
Part~(a) (instability) is proved directly in the main text. Given the fluid limit technique developed in \cref{sec:limit}, Part~(a) also easily follows from Theorem 3.1 and Proposition 5.1 of \citet{Dai1996}.
We restate the load condition \cref{eq:load<} here for convenience:
\begin{align}\label{eq:load-a}
  \lambda(m_p+m_d)< b_{\max}/t_{b_{\max}},
\end{align}
where $b_{\max}$ is the token budget, $m_p$ and $m_d$ are the mean prefill and decode
token sizes, $t_{b_{\max}}$ is the processing time of a full
batch, and $\lambda$ is the request arrival rate. The appendix is organized as follows. In \cref{sec:dynamics}, we introduce the dynamics of the processing model.
In \cref{sec:fm}, we define the fluid model and prove it is stable under \cref{eq:load-a}. In \cref{sec:limit}, we introduce fluid limits that justify the fluid model equations. In \cref{sec:stability}, we complete the proof of \cref{thm:main}.

While a simpler, self-contained Lyapunov proof of Part~(b) (stability) is also given in \cref{sec:proof-thm-main}, the fluid limit technique developed here is more flexible and extends naturally to the network settings of \cref{sec:AI-agent}. Readers are referred to~\citep{dai2020processing} for a detailed treatment of the fluid limit technique for stochastic processing networks.

\subsection{System dynamics and scheduling algorithms}
\label{sec:dynamics}
To study the dynamics of the processing model of the LLM server, we introduce aggregate quantities derived from the per-request state $\{(P_i(n), D_i(n)) : i \in \mathcal{Q}_n\}$ defined in \cref{sec:markov-chain}.
Let $Z_p(n)$ and $Z_d(n)$ be the number of requests in the prefill and decode phases, respectively, at the end of time slot $n$:
\begin{align*}
  Z_p(n) = \abs{\{i \in \mathcal{Q}_n : P_i(n) > 0\}}, \qquad Z_d(n) = \abs{\{i \in \mathcal{Q}_n : P_i(n) = 0\}}.
\end{align*}
Let $W_p(n)$ and $W_d(n)$ be the total number of remaining prefill and decode tokens, respectively:
\begin{align*}
  W_p(n) = \sum_{i \in \mathcal{Q}_n} P_i(n), \qquad W_d(n) = \sum_{i \in \mathcal{Q}_n} D_i(n).
\end{align*}
Note that $W_p(n) + W_d(n) = \abs{X(n)}$, the total unprocessed tokens in \cref{eq:diverge}.

Assume a batch is completed at $n$ and is ready to load the next
batch to be started at time $n+1$. The LLM server needs a scheduling algorithm
to decide which tokens to go into the batch.

The following scheduling algorithms can be
used to load the next batch. Using newly introduced notation, we recap the
definitions of two families of scheduling algorithms.

\noindent \textbf{Work-conserving scheduling algorithms.}
When
\begin{align}\label{eq:conserv}
  W_p(n)+ Z_d(n) \ge b_{\max},
\end{align}
the next batch is a full batch (token load $b = b_{\max}$).

To describe the dynamics, we further introduce the quantities
$B_f(n)$, $E(n)$, $F_f(n)$, and $V_f(N)$.  For each phase
$f\in \{p, d\}$, define $B_f(n)$ to be the cumulative \textit{number of
  tokens} that have been completed phase $f$ processing by time $n$;
define $F_f(n)$ to be the cumulative \textit{number of requests} that
have completed phase $f$ processing by time $n$; define
$E(n) = \sum_{\ell=1}^n a_\ell$ to be the cumulative number of requests that
have arrived by $n$, and 
 \begin{align}
   \label{eq:V_p}
   & V_f(N) = \sum_{i=1}^N v_f(i)
 \end{align}
 to be the total number of phase $f$ tokens brought in by the first $N$ requests. It follows that
\begin{align}
  & Z_p(n) = Z_p(0)+ E(n)-F_p(n)\ge 0,\label{eq:Z1} \\
  & Z_d(n) = Z_d(0)+ F_p(n)-F_d(n)\ge 0, \label{eq:Z2}\\
  & W_p(n)= V_p(Z_p(0)+E(n)) - B_p(n) \ge 0,\label{eq:W1} \\
  & W_d(n)= V_d(Z_d(0)+F_p(n)) - B_d(n)\ge 0,\label{eq:W2}\\
  & B_f\big(n_2 t_{b_{\max}}\big)-B_f\big(n_1t_{b_{\max}}\big) \le (n_2-n_1) b_{\max}, \nonumber \\
  & \qquad \text{ for each } f\in \{p, d\}, \quad n_1, n_2\in \N \text{ with } n_1< n_2.\label{eq:B}
\end{align}
To establish the fluid limit in \cref{sec:limit}, we need to show that requests complete each phase in approximately arrival order (bounded overtaking). Recall that the token sizes are bounded: $v_p(i)\le v_p^{\max}$ and $v_d(i)\le v_d^{\max}$ for all $i$. Writing $v^{\max}\coloneqq v_p^{\max}+v_d^{\max}$, the following lemma shows that the $K$-FCFS condition together with bounded token sizes guarantees this property.

\begin{lemma}[Bounded overtaking in phase completion order]\label{lem:bounded-overtaking}
Under any $K$-FCFS scheduling algorithm, there exists a finite constant $K'$ (depending on $K$, $v^{\max}$, and $b_{\max}$) such that for each phase $f\in\{p,d\}$, request~$j$ can complete phase~$f$ only if every request~$i\le j-K'$ has already completed phase~$f$.
\end{lemma}
\begin{proof}
The key observation is: whenever any request $j'\ge i+K$ processes a decoding token ($\delta_{j'}^d=1$), the $K$-FCFS condition forces $\delta_i^p+\delta_i^d\ge 1$, so request~$i$ makes at least one token of progress. Since $i$ requires at most $v^{\max}$ tokens in total to depart, it departs after at most $v^{\max}$ such batches.

It remains to count how many later requests can overtake~$i$ within these $v^{\max}$ batches. In each batch, at most $b_{\max}$ requests decode simultaneously (each contributing one decode token to the budget), so at most $b_{\max}$ requests with index $\ge i+K$ can depart per batch. Over $v^{\max}$ batches, at most $v^{\max}\cdot b_{\max}$ later requests overtake~$i$, giving $K'\le K+v^{\max}\cdot b_{\max}$ for $f=d$ (departure order).

For $f=p$ (prefill completion order): when $i$ is still in prefill ($P_i(n)>0$), feasibility constraint \cref{eq:prefill-pre-decode} forces $\delta_i^d=0$, so the guaranteed token is a prefill token ($\delta_i^p\ge 1$). Since $i$ has at most $v_p^{\max}$ prefill tokens, the same counting gives $K'\le K+v_p^{\max}\cdot b_{\max}$.
\end{proof}

With \cref{lem:bounded-overtaking} in hand, the cumulative tokens processed $B_f(n)$ are tightly coupled with the cumulative tokens brought by completed requests $V_f(F_f(n))$. Specifically, the bounded overtaking constant $K'$ from \cref{lem:bounded-overtaking} yields
\begin{align}\label{eq:FCFS}
    V_f\big(F_f(n)-K'\big) \le B_f(n) \le V_f\big(F_f(n)+K'\big), \quad f\in \{p, d\}.
\end{align}
For the upper bound: by \cref{lem:bounded-overtaking}, the only requests that can have any phase-$f$ tokens processed by time~$n$ are those with arrival index at most $F_f(n)+K'$, so $B_f(n)\le V_f(F_f(n)+K')$. For the lower bound: all requests with arrival index $\le F_f(n)-K'$ must have completed phase~$f$ by time~$n$ (again by bounded overtaking), so their full token loads are included in $B_f(n)$, giving $B_f(n)\ge V_f(F_f(n)-K')$.

This is the analog of Key Relationship (6.51) in~\citep{dai2020processing}. Since $K'$ is a finite constant, it vanishes under fluid scaling.

\subsection{The  LLM  fluid model \& fluid model calculus}
\label{sec:fm}

In this section, we introduce the fluid model of the LLM
server.  The fluid model is defined through a set of fluid
model equations including some inequalities. These equations will need to be justified through a fluid limit procedure to be explained in Section~\ref{sec:limit}. Fix any  $K$-FCFS scheduling algorithm.
These fluid model equations include: for any time $t\in \R_+\equiv[0, \infty)$,
\begin{align}
  &  Z_p(t)=Z_p(0)+ \lambda t - F_p(t), \label{eq:f1} \\
  &  Z_d(t)= Z_d(0)+  F_p(t) - F_d(t),\label{eq:f2} \\
  & B_f(0)=0, \qquad 0\le  B_f(t )-B_f(s) \le (t-s) b_{\max}/t_{b_{\max}}, \quad 0\le  s\le  t,
    \quad f\in\{p, d\}.\label{eq:f3}\\
  &   F_f(t) = \frac{1}{ m_f }B_f(f), \quad f\in \{p, d\}, \label{eq:f4} \\
  &   W_f(t) = m_f Z_f(t), \quad f\in \{p, d\}.   \label{eq:f5}   
\end{align}
Here, we intentionally overload the notational system in Section~\ref{sec:dynamics} to emphasize the similarity of the fluid analog. Each overloaded function is differentiated by its argument $n\in \N$
and $t\in \R_+$.
Fluid model equation (\ref{eq:f3}) says that the function $B_f(\cdot)$
is Lipschitz continuous with Lipschitz constant $b_{\max}/t_{b_{\max}}$.  Fluid model
equation (\ref{eq:f4}) implies that function $F_f(\cdot)$ is Lipschitz
continuous. Fluid model equations (\ref{eq:f1})-(\ref{eq:f2}) and (\ref{eq:f5}) imply that
$Z_f(\cdot)$ and $W_f(\cdot)$ are Lipschitz continuous as well. In conclusion,
the multidimensional function
\begin{align}\label{eq:f6}
  \big(B_f(\cdot), F_f(\cdot), W_f(\cdot), Z_f(\cdot), f\in \{p, d\}\big) 
\end{align}
is Lipschitz continuous. Therefore, from real analysis (c.f. Lemma A.2 of~\citep{dai2020processing}. and its commentary), this function is 
absolutely continuous. As a consequence, it is differential almost surely  everywhere, and for any component  $x(\cdot)$ in (\ref{eq:f6}),
\begin{align}\label{eq:thmcal}
  x(t) -x(s) = \int_s^t \dot x(u)du, \quad s< t,
\end{align}
where $\dot x(u)$ denotes the derivative of the function $x(\cdot)$ at time $u$.
Equation (\ref{eq:thmcal}) is simply the fundamental theorem of calculus when $x(\cdot)$ is continuously differentiable. In general,
the integral on the right of (\ref{eq:thmcal}) is interpreted as the Lebesgue integral, and it is sufficient that  $\dot x(u)$ is well defined almost everywhere. See, for example,
Lemma A.3 of~\citep{dai2020processing}. and the references there.
\begin{definition}
  A point $t>0$ is said to be a \textit{regular point} for a fluid model solution (\ref{eq:f6}) if the
  solution is differentiable at time $t$. 
\end{definition}
 When the fluid model solution is clear in a context, we simply say a time $t$ is a regular point without reference to the fluid model solution.
When the $K$-FCFS scheduling algorithm  is also work-conserving, the fluid model equations include: for any regular time $t>0$,
\begin{align}\label{eq:f7}
  W_p(t)+W_d(t) >0 \text{ implies } \dot B_p(t) + \dot B_d(t) = b_{\max}/t_{b_{\max}}.
\end{align}

\begin{definition}
  A function in (\ref{eq:f6}) is said to be a  \textit{fluid model solution} if it satisfies fluid model equation (\ref{eq:f1})-(\ref{eq:f5}),
  a \textit{work-conserving fluid model solution} if, in addition, it
  satisfies (\ref{eq:f7}).
\end{definition}

\begin{definition}\label{def:fm-stable}
  The fluid model is said to be \textit{stable}
if there exists a time $\delta>0$  such that
  $Z(t)=0$ for $t\ge \delta$  for any fluid model solution with $\abs{Z(0)}\equiv Z_p(0)+Z_d(0)\le 1$.
\end{definition}

\begin{proposition}\label{pro:1}
  Under any work-conserving, $K$-FCFS algorithm, the fluid
  model is stable under load condition~\cref{eq:load-a}.
\end{proposition}
\begin{proof}
  Given any fluid model solution $(B(\cdot), F(\cdot), W(\cdot),Z(\cdot))$,
  we use the Lyapunov function
  \begin{align}
    \label{eq:laypunov}
    f(t) = (m_p+m_d) Z_p(t) + m_d\, Z_d(t)
  \end{align}
  as the total workload in the system at time $t$: each prefill request still requires $(m_p+m_d)$ tokens of processing (prefill plus future decode), while each decode request requires $m_d$ tokens.
    It follows from fluid model equations (\ref{eq:f1})-(\ref{eq:f5}) that
  \begin{align*}
    f(t) = f(0) + \lambda (m_p+m_d) t - \big(B_p(t)+B_d(t)\big).
  \end{align*}
For each regular point $t>0$, $f(t)>0$ implies that
  $\dot f(t)=\lambda(m_p+m_d)-b_{\max}/t_{b_{\max}}$, which is negative by load condition~\cref{eq:load-a}. It follows from Lemma 8.5 of~\citep{dai2020processing} that $f(t)=0$ for $t\ge f(0)/\delta$, where
  \begin{align*}
    \delta = b_{\max}/t_{b_{\max}} - \lambda(m_p+m_d)>0.
  \end{align*}
\end{proof}

\subsection{Fluid limit}
\label{sec:limit}
Fix a work-conserving $K$-FCFS scheduling algorithm.
In this section, we define fluid limits and prove that each fluid limit
is a fluid model solution satisfying (\ref{eq:f1})-(\ref{eq:f7}).  For
each state $x$ in the space $\mathcal{X}$ in Section 2.2, we use $z_p$
and $z_d$ to denote the corresponding token counts in prefill phase
and decode phase, respectively. Define $\abs{x}=z_p+z_d$ to be the total
number of tokens in state $x$.  Assume that two iid sequences in
(2.1) are defined on some probability space
$(\Omega, \mathcal{F}, \mathbb{P})$. Since the token sizes are bounded ($v_f(i)\le v_f^{\max}$), all moments are finite:
\begin{align}
  \label{eq:m1+}
  \E\Big[(v_f(1))^{1+\epsilon}\Big]<\infty \quad \text{for all } \epsilon>0,\quad f\in \{p, d\}.
\end{align}
In particular,
\begin{align}\label{eq:vmax}
  \frac{1}{N} \max_{1\le i\le N}\big(v_p(i)+v_d(i)\big) \le \frac{v^{\max}}{N} \to 0 \quad \text{as } N\to\infty.
\end{align}

Recall the definitions in
Section~\ref{sec:basic} of $B_f(n)$, $E(n)$, $F_f(n)$, $W_f(n)$, and
$Z_f(n)$ for each $n\in\N$ and each phase $f\in \{p, d\}$.  For each
sample path $\omega\in \Omega$, one has realization of two sequences
$\{a_n(\omega),n\in\N\}$ and
$\{(v_p(i,\omega), v_d(i,\omega)), i\in\N\}$. With these two
sequences, the given initial state $x\in \mathcal{X}$, and any fixed
$K$-FCFS scheduling algorithm, one can construct the corresponding
realization  $B^x_f(n,\omega)$, $E(n, \omega)$, $F^x_f(n,\omega)$, $W^x_f(n,\omega)$, and
$Z^x_f(n,\omega)$ for each $n\in\N$ and each $f\in \{p,d\}$. For each $\omega\in \Omega$, each time $t\ge 0$, and each state $x\in \mathcal{X}$ with $\abs{x}>0$, define fluid scaled quantities
\begin{align*}
&  \hat E^x(t, \omega)=\frac{1}{\abs{x}} E(\abs{x}t, \omega), \quad
  \hat B^x_f(t, \omega)=\frac{1}{\abs{x}} B^x_f(\abs{x}t, \omega), \quad
 \hat F^x_f(t, \omega)=\frac{1}{\abs{x}} F^x_f(\abs{x}t, \omega),\\
&  \hat W^x_f(t, \omega)=\frac{1}{\abs{x}} W^x_f(\abs{x}t, \omega),\quad
  \hat Z^x_f(t, \omega)=\frac{1}{\abs{x}} Z^x_f(\abs{x}t, \omega),
\end{align*}
where, whenever $\abs{x}t$ is not an integer, it is applied
the floor operation to make it an integer.
Following the SLLN theorem,
\begin{align}\label{eq:slln}
  \Prob\Big\{\omega\in \Omega: \lim_{n\to\infty}\frac{1}{n} E(n, \omega) =\lambda, \quad  \lim_{N\to\infty} \frac{1}{N} V_p(N,\omega)=m_p, \quad
   \lim_{N\to\infty} \frac{1}{N} V_d(N,\omega)=m_d\Big\} =1.
\end{align}
Denote the set $\Omega_0\subset \Omega$ in which the limits in (\ref{eq:slln})
and (\ref{eq:vmax}) exist. Clearly $\Prob\{\Omega_0\}=1$. We now define fluid limits for each
sample path $\omega\in \Omega_0$.
\begin{lemma}
  For each $\omega\in \Omega_0$ and each unbounded set $A \subseteq \mathcal{X}$, there exists a sequence $\{x_k\}\subset A$ with $\abs{x_k}\to\infty$ as $k\to\infty$, and  a function $(\hat B_p(\cdot), \hat B_d(\cdot))$ such that
  \begin{align}
    \label{eq:flB}
    \lim_{k\to\infty} \sup_{0\le t\le T} \abs{B^{(x_k)}_f(t,\omega)-\hat{B}_f(t)} = 0
    \quad \text{for each } T>0 \text{ and } f\in \{p,d\}.
  \end{align}
\end{lemma}
The proof follows from the Lipschitz property (\ref{eq:B}) and is identical to the proof of (12.37) on page 285 of~\citep{dai2020processing}. The
convergence mode in  (\ref{eq:flB}) is known as the uniform convergence on compact sets or u.o.c.\ convergence. Since $\abs{\hat Z^{(x_k)}(0,\omega)}=1$,
the sequence $\{\hat Z^{(x_k)}(0,\omega), k\ge 1\}$ is in a compact set of $\R^2_+$. By taking a subsequence if needed,  we will assume 
\begin{align}
  \label{eq:initial}
  Z^{(x_k)}(0, \omega) \to z=(z_p, z_d)
\end{align}
for some $z\in \R^2_+$ for the sequence of $x_k$ in (\ref{eq:flB}).

\begin{lemma}
  Under any $K$-FCFS scheduling algorithm, fix
  a $\omega\in \Omega_0$ and a sequence  $\{x_k\}\subset \mathcal{X}$ such that
  (\ref{eq:flB}) and (\ref{eq:initial}) hold. For each $f\in \{p, d\}$,
  \begin{align}\label{eq:fl}
    \Big(\hat F^{x_k}_f(\cdot, \omega), \hat W^{x_k}_f(\cdot, \omega),
    \hat Z^{x_k}_f(\cdot, \omega)\Big), 
    \to \Big(\hat F_f(\cdot), \hat W_f(\cdot), \hat Z_f(\cdot)\Big) \quad u.o.c.\ \quad
    \text{ as } k\to\infty,
  \end{align}
  where
  \begin{align}
    & \hat F_f(t) = \frac{1}{m_f}\hat B_f(t),\label{eq:fl1}\\
    & \hat Z_p(t)=z_p + \lambda t -\hat F_p(t), \label{eq:fl2}\\
    & \hat Z_d(t)=z_d +\hat F_p(t) - \hat F_d(t),\label{eq:fl3} \\
    & \hat W_f(t) = m_f\hat Z_f(t).\label{eq:fl4}
  \end{align}
\end{lemma}
\begin{proof}
  To prove (\ref{eq:fl}), it is sufficient to prove the u.o.c.\ convergence for each component. For $\hat F_f^{x_k}(\cdot, \omega)$, we utilize relationship (\ref{eq:FCFS}). The proof is given by Lemma 6.8 of~\citep{dai2020processing}. and fluid model equation (\ref{eq:fl1}) is identical to (6.43) of~\citep{dai2020processing}. The rest of the proof is easy to complete, following (\ref{eq:Z1})-(\ref{eq:W2}) and analogous to the proof of Theorem 12.13 in~\citep{dai2020processing}.
\end{proof}

\begin{definition} \label{def:fl}
  Functions $(\hat B_f(\cdot), \hat F_f(\cdot), W_f(\cdot), Z_d(\cdot))$ is called a \textit{fluid limit} if there exists $\omega\in \Omega_0$ and a sequence of initial states  $\{x_n\}\subset \mathcal{X}$ such that $\abs{x}_k\to\infty$ and limits (\ref{eq:flB})-(\ref{eq:fl}) hold.
\end{definition}

\begin{lemma}\label{lem:fl}
Under any $K$-FCFS work-conserving scheduling algorithm, each fluid limit satisfies fluid model equations (\ref{eq:f1})-(\ref{eq:f7}).
\end{lemma}
\begin{proof}
  The fluid limit satisfies fluid model equations (\ref{eq:f1})-(\ref{eq:f6}) is covered in (\ref{eq:fl1})-(\ref{eq:fl4}). It suffices to prove that each work-conserving fluid limit satisfies (\ref{eq:f7}). The proof of the latter is analogous to the proof of Theorem 7.2 of~\citep{dai2020processing}.
\end{proof}

\subsection{Proof of Theorem 1}
\label{sec:stability}
\begin{definition}
  Under a fixed scheduling algorithm, the fluid limits are said to be
  \textit{stable} if there exists a constant $\delta>0$ such that
  \begin{align}
    \label{eq:flstable}
    \hat Z_f(t) =0 \quad t\ge \delta
  \end{align}
  for each fluid limit $(\hat B, \hat F, \hat W, \hat Z)$ defined in Definition~\ref{def:fl}.
\end{definition}

\begin{proposition}\label{pro:2}
  Fix a $K$-FCFS work-conserving scheduling algorithm.
  If the fluid limits are stable, then the corresponding DTMC in Section 2.2 is positive recurrent.
\end{proposition}
\begin{proof}
  The proof is analogous to the proof of Theorem 6.2 of~\citep{dai2020processing}.  The latter
  theorem is stated in the setting of continuous time Markov
  chains. Assume the fluid limits are stable.  The Lyapunov function
  used in Lemma 3.7 of~\citep{dai2020processing}. in the continuous setting can be copied
  verbatim in the current setting;  see the proof of 
   Theorem 12.27 in~\citep{dai2020processing}., which is for the slotted discrete time model.
  In the proof of Theorem 6.2 in~\citep{dai2020processing}.,
  one needs to verify that 
  the sequence of random variables
  \begin{align}\label{eq:ui}
\left\{ \frac{E(n)}{n}, n\ge 1\right\} \text{   is uniformly integrable.}
  \end{align}
  In~\citep{dai2020processing}, Proposition B.5 is cited to prove (\ref{eq:ui}). The proof of
  Proposition B.5 there utilized the finiteness of the second moment of
  $a_n$ in (2.1).  Examining the proof of
  Lemma 4.5 of~\citep{dai1995positive}, one concludes that  Proposition B.5 continues to
  hold under the first moment assumption on $a_n$ as in (2.2).
\end{proof}

\begin{proof}[Proof of \cref{thm:main} Part~(b)]
  Part~(a) (instability) has been proved directly in the main paper. It remains to prove Part~(b) (stability).
  By \cref{lem:fl}, each fluid limit
  $(\hat B, \hat F, \hat W, \hat Z)$ satisfies $\abs{Z(0)}=1$ and
  fluid model equations (\ref{eq:f1})-(\ref{eq:f7}).
  The fluid model has been shown stable in \cref{pro:1}
  under the load condition~\cref{eq:load-a}. Therefore the fluid limits are stable. By \cref{pro:2}, the DTMC is positive recurrent.
\end{proof}

\subsection{A Sketch of the Proof of Proposition~\ref{pro:parallel}}\label{sec:proof-pro-parallel}

  \begin{proof}
  The fluid limit technique for proving Theorem~1 continues to apply
  here. Let $(Z^k_p(t), Z_d^k(t))$ be the fluid request level at
  server $k$.  Let $A_k(t)$ be the cumulative amount of fluid request
  to server $k$ by time $t$. By SLLN, 
  \begin{align}
      \sum_{k} A_k(t) = \lambda t, \quad t\ge 0.
  \end{align}
  See, e.g, the proof of (11.13) of \cite{dai2020processing}.
  Define the (total) workload for server $k$
  \begin{align*}
    f^k(t) = (m_p+m_d) Z^k_p(t)+ m_d\, Z^k_d(t).
  \end{align*}
  It follows that 
  \begin{align*}
      f^k(t)=f^k(0) + A_k(t)(m^k_p+m^k_d)-(B^k_p(t)+B^k_d(t)) \quad t\ge 0.
  \end{align*}
  and for each regular point $t$,
  \begin{align}\label{eq:conserve}
      f^k(t)>0 \quad \text{implies}\quad \dot B^k_p(t)+\dot B^k_d(t)=b_{\max}/t_{b_{\max}}.
  \end{align}
Furthermore, it can be proved by SLLN that 
under the random assignment load-balancing algorithm,  
\begin{align}\label{eq:ra}
  A^k(t)=(\lambda/K) t, \quad t\ge 0,  \quad k=1, \ldots, K, 
\end{align}
and 
under join-the-shortest-queue (JSQ), at each
  regular point $t$, 
\begin{align}\label{eq:jsq}
 \dot A_k(t)=0 \quad \text{whenever} \quad   Z^k_p(t)+Z^k_d(t) > \min_{k'\neq k}(Z^{k'}_p(t)+Z^{k'}_d(t)).
\end{align}
Define the total system workload
  \begin{align*}
    F(t)=\sum_{k=1}^K f^k(t)
  \end{align*}
  to be the Lyapunov function.
We claim that there exists $\delta>0$ such that  at each regular point $t>0$ with $F(t)>0$, 
\begin{align} \label{eq:jsqdrift}
    \dot F(t)\le -\delta.
\end{align}
\cref{eq:jsqdrift} implies that $F(t)=0$ for $t\ge F(0)/\delta$
and the fluid model is stable.

 The rest of this proof is devoted to prove \cref{eq:jsqdrift}. 
  Let $\mathcal{M}(t)$ denote the set  of  servers with positive workload, i.e., those with $f^k(t)>0$.
  Since $t$ is a regular point, $\dot f^k(t)=0$ for each $k\not\in \mathcal{M}(t)$. Thus,
\begin{align*}
    \dot F(t)=\sum_{k\in \mathcal{M}(t)} \dot f^k(t)
  \end{align*}
  
  \noindent\textit{Proof \cref{pro:parallel} under random assignment algorithm}\;
  For each $k\in \mathcal{M}(t)$, it follows from \cref{eq:ra} and \cref{eq:conserve} that 
  \begin{align*}
    \dot f^k(t)= (m_p+m_d)\lambda/K- b_{\max}/t_{b_{\max}} <0.
  \end{align*}
  Since $F(t)>0$, the set $\mathcal{M}(t)$ is nonempty. Therefore, \cref{eq:jsqdrift} holds with 
  $\delta=b_{\max}/t_{b_{\max}} -(m_p+m_d)\lambda/K$.

 \noindent\textit{Proof \cref{pro:parallel} under JSQ algorithm}\; 
 If $\mathcal{M}(t)=\{1, \ldots, K\}$,
\begin{align*}
    \dot F(t) = (m_p+m_d)\lambda - K\cdot b_{\max}/t_{b_{\max}}.
  \end{align*}
If $\mathcal{M}(t) \neq \{1, \ldots, K\}$, there exists server $\ell$ whose
workload is zero. Therefore, we have $Z_p^\ell(t)=Z_d^\ell(t)=0$. For each 
server $k\in \mathcal{M}(t)$,  $Z_p^k(t)+Z_d^k(t)>0$. It follows from \cref{eq:jsq} that $\dot A_k(t)=0$ for each $k\in \mathcal{M}(t)$, which together with \cref{eq:conserve} implies 
\begin{align*}
    \dot F(t) = - \abs{\mathcal{M}(t)}b_{\max}/t_{b_{\max}}
    \le - b_{\max}/t_{b_{\max}}.
\end{align*}
Therefore, we have proved \cref{eq:jsqdrift} with 
\begin{align*}
    \delta = \min\Big( b_{\max}/t_{b_{\max}},  K\cdot b_{\max}/t_{b_{\max}}- (m_p+m_d)\lambda\Big).
\end{align*}
  \end{proof}

%% file: sections/appendix_c_lyapunov_proof.tex
\section{Direct Lyapunov Proof of \Cref{thm:main} Part~(b)}\label{sec:proof-thm-main}

This appendix provides an alternative, self-contained proof of Part~(b) (stability) of \cref{thm:main} via the Foster--Lyapunov criterion, without invoking the fluid limit machinery. Unlike the fluid limit proof in \cref{sec:basic}, this argument uses only the work-conserving property and does not require the $K$-FCFS condition. However, it does not extend to the network settings of \cref{sec:AI-agent}.

\begin{proof}
    Assume $b_{\max} - \lambda (m_p+m_d) = \epsilon'.$ Let $\epsilon = \min(\epsilon', L\lambda)$, where $L := \E[v_d(1)^2]$. We construct the following Lyapunov function:
    \begin{align*}
        f(X) = \sum_{i \in [Q(X)]} \left(p_{i}' + d_{i}' + \frac{\epsilon (1(p_{i}'>0)+d_{i}')^2}{4L\lambda}\right).
    \end{align*}
    Let $\pi(X) = (\delta_{i}^{p}, \delta_{i}^{d})_{i\in Q(X)}$. Note that $1(p_{i}'>0)+d_{i}'$ will decrease by 1 only when $\delta_{i}^{p}=p_{i}'$ or $\delta_{i}^{d}=1$. Thus the decrease of the Lyapunov function is
    \begin{align*}
    	\Delta(X) = \sum_{i \in [Q(X)]} \left((\delta_{i}^{p} + \delta_{i}^{d}) + \frac{\epsilon (2(d_{i}^{'}+1(p_i'>0))-1)}{4L\lambda} 1(\delta_{i}^{p}=p_{i}' \text{ or } \delta_{i}^{d}=1) \right).
    \end{align*}
   Note that for the new arrivals, we have $\E[p_l+d_l] = m_p+m_d$ and $\E[(d_l+1)^2] \leq 2\E[d_l^2] \leq 2L$. Hence, the one-step change of the Lyapunov function is bounded by
    \begin{align*}
    \E[f(X(n+1))] - f(X(n)) \leq  \lambda (m_p+m_d) + \frac{2\epsilon L\lambda }{4L\lambda } - \Delta(X(n)).
\end{align*}
    We aim to show that once $f(X(n)) \geq B$ for a large enough $B$, we have
    \begin{align}
    \E[f(X(n+1))] - f(X(n)) \leq \lambda (m_p+m_d) + \frac{\epsilon}{2} - \Delta(X(n)) \leq - \frac{\epsilon}{2}. \label{eq:Lyapunov}
    \end{align}

If so, we can invoke the classical Foster--Lyapunov criterion (Theorem 4.3.1 in \citep{xu2023drift}) to conclude that $X(n)$ is positive recurrent.

To show \cref{eq:Lyapunov} holds, note that if $\sum_{i \in [Q(X)]} p_{i}' +  \sum_{i \in [Q(X)]} \mathbf{1}(p_{i}'=0) \geq b_{\max}$, then
$$
\Delta(X(n)) \geq  \sum_{i \in [Q(X)]} (\delta_{i}^{p} + \delta_{i}^{d}) = b_{\max}
$$
by the work-conserving policy. Thus \cref{eq:Lyapunov} holds:
\begin{align*}
 \E[f(X(n+1))] - f(X(n)) \leq \lambda (m_p+m_d) + \frac{\epsilon}{2} - b_{\max} \leq  \frac{\epsilon}{2}  - \epsilon' \leq -\frac{\epsilon}{2}.
\end{align*}

Now consider $\sum_{i \in [Q(X)]} p_{i}' + \sum_{i \in [Q(X)]} \mathbf{1}(p_{i}'=0) < b_{\max}$, and $f(X(n)) \geq B > b_{\max}$. In this setting, every request $i$ in the queue will have either $\delta_{i}^{p}=p_{i}'$ or $\delta_{i}^{d}=1$; otherwise the work-conserving assumption is violated.

The key intuition is: when $f(X(n))$ is large enough, there exists a request with sufficiently large decoding length $\tilde{d}$. This request will be served and the decrease $\Delta(X)$ will be lower bounded by $\frac{\epsilon(2\tilde{d}-1)}{4L\lambda} \geq b_{\max}$, which suffices for \cref{eq:Lyapunov}.

In particular, let $B := 2b_{\max}\left(\frac{4L\lambda}{\epsilon}b_{\max} \right)^2 + b_{\max}$ and $\tilde{d} := \max_{i \in [Q(X)]} d_i'.$ We have
\begin{align*}
f(X(n)) \geq B
&\implies \sum_{i \in [Q(X)]}  \left(p_{i}' + d_{i}' + \frac{\epsilon (d_i'+1)^2}{4L\lambda}\right) \geq B\\
&\overset{(i)}{\implies}  \sum_{i \in [Q(X)]}  d_{i}' + \frac{\epsilon (d_i')^2}{L\lambda} \geq B - b_{\max}\\
&\overset{(ii)}{\implies} \tilde{d} + \frac{\epsilon \tilde{d}^2}{L\lambda} \geq \frac{(B-b_{\max})}{b_{\max}}\\
&\overset{(iii)}{\implies} 2\tilde{d}^2 \geq \frac{(B-b_{\max})}{b_{\max}}\\
&\implies \tilde{d} \geq \frac{4L\lambda}{\epsilon} b_{\max}
\end{align*}
where (i) uses $\sum_{i \in [Q(X)]} p_{i}' < b_{\max}$ and $d_i'+1 \leq 2d_{i}'$ for $d_i'\geq 1$; (ii) uses $Q(X) < b_{\max}$ (since each request contributes at least one token) and the pigeonhole principle; (iii) uses $d_{i}' \leq (d_{i}')^2$ and $\epsilon \leq L\lambda$.

Next, $\tilde{d}$ provides a lower bound for $\Delta(X)$ because all requests satisfy $\delta_{i}^{p}=p_{i}'$ or $\delta_{i}^{d}=1$:
\begin{align*}
\Delta(X(n)) \geq   \frac{\epsilon (2\tilde{d}-1)}{4L\lambda} \geq \frac{\epsilon \tilde{d}}{4L\lambda}.
\end{align*}

Then when $f(X(n)) \geq B$, we have
\begin{align*}
  \Delta(X(n)) \geq  \frac{\epsilon \tilde{d}}{4L\lambda} \geq   \frac{\epsilon}{4L\lambda} \frac{4L\lambda}{\epsilon} b_{\max} = b_{\max}.
\end{align*}
Thus \cref{eq:Lyapunov} holds and the proof is complete.
\end{proof}

%% file: sections/appendix_e_fork_join_proof.tex
\section{Proof of \Cref{thm:fork-join} (Fork-Join Network)}
\label{sec:proof-fork-join}

This appendix proves \cref{thm:fork-join}. The proof builds on the fluid limit machinery of \cref{sec:basic} and applies the \emph{bounded overtaking} property of $K$-FCFS scheduling (\cref{lem:bounded-overtaking}) to control the join buffer sizes.

\subsection{Network model and notation}

Consider $k+1$ LLM servers $q_1,\ldots,q_{k+1}$ arranged in a fork-join topology: requests arrive at server~$q_1$, fork to servers~$q_2,\ldots,q_k$ in parallel, and rejoin at server~$q_{k+1}$ after all sub-tasks complete.

\begin{description}
  \item[$E(n)$] Cumulative external arrivals by time $n$.
  \item[$D_s(n)$] Cumulative departures from server~$q_s$ by time~$n$, for $s \in \{1,\ldots,k+1\}$.
  \item[$J_i(n)$] Join buffer for server~$q_i$: the number of requests whose sub-task at $q_i$ has completed but that are still waiting for other sub-tasks, for $i\in\{2,\ldots,k\}$.
  \item[$R(n)$] Cumulative releases from the join by time~$n$ (= cumulative arrivals to $q_{k+1}$).
\end{description}

The flow-balance equations are:
\begin{align}
Q_1(n) &= Q_1(0) + E(n) - D_1(n), \label{eq:fj-q1-gen}\\
Q_s(n) &= Q_s(0) + D_1(n) - D_s(n), \quad s \in \{2,\ldots,k\}, \label{eq:fj-qs-gen}\\
J_i(n) &= J_i(0) + D_i(n) - R(n), \quad i \in \{2,\ldots,k\}, \label{eq:fj-ji-gen}\\
Q_{k+1}(n) &= Q_{k+1}(0) + R(n) - D_{k+1}(n). \label{eq:fj-qk1-gen}
\end{align}

\subsection{Bounded overtaking and join buffer bound}

\begin{lemma}[Join buffer bound]\label{lem:join-bound-gen}
Under $K$-FCFS scheduling at each fork server $q_2,\ldots,q_k$, the total join buffer is uniformly bounded:
\[
\sum_{i=2}^{k} J_i(n) \le 2K'(k-1) \quad \text{for all } n\ge 0,
\]
where $K'$ is the bounded overtaking constant from \cref{lem:bounded-overtaking}.
\end{lemma}

\begin{proof}
Let $C_i(r)$ denote the position of request~$r$ in the completion order at fork server~$q_i$. By \cref{lem:bounded-overtaking}, $|C_i(r)-r|\le K'$ for each $i$ and~$r$, since all sub-tasks of a given request arrive at the fork servers in the same order (all forked from~$q_1$ in arrival order).

For any two fork servers~$q_i$ and~$q_\ell$, the triangle inequality gives $|C_i(r)-C_\ell(r)|\le 2K'$. Request~$r$ is released from the join when the last fork server completes its sub-task, i.e., when $\max_{i} C_i(r)$ completions have occurred. At the time request~$r$'s sub-task at~$q_i$ completes, at most $2K'$ additional completions at any other fork server~$q_\ell$ are needed before $r$ is released.

To bound the total join buffer: at any time~$n$, each ``waiting'' request~$r$ in $J_i(n)$ has completed at $q_i$ but is waiting for some $q_\ell$. The number of such requests is bounded by the gap $|C_i(r)-C_\ell(r)|\le 2K'$ between completion orders. Summing over all $k-1$ fork servers gives the bound $2K'(k-1)$.
\end{proof}

\subsection{Fluid limits}

Define fluid-scaled processes for initial state~$x$ with $|x|=\sum_{s=1}^{k+1}Q_s(0)+\sum_{i=2}^k J_i(0)$:
\[
\hat E^x(t) = \frac{E(|x|t)}{|x|}, \quad
\hat D_s^x(t) = \frac{D_s(|x|t)}{|x|}, \quad
\hat Q_s^x(t) = \frac{Q_s(|x|t)}{|x|}, \quad
\hat J_i^x(t) = \frac{J_i(|x|t)}{|x|}, \quad
\hat R^x(t) = \frac{R(|x|t)}{|x|}.
\]

\begin{lemma}[Fluid limit properties]\label{lem:fj-fluid-gen}
With probability one, for each unbounded sequence of initial states, there exists a subsequence along which all scaled processes converge uniformly on compact sets to Lipschitz-continuous limit functions
\begin{align*}
    (\hat E, \hat D, \hat Q, \hat J, \hat R)
\end{align*}
with
\begin{align}
\hat E(t) & =\lambda t \\
\hat J_i(t) &= 0, \quad i=2,\ldots,k, \label{eq:fj-fl-vanish}\\
\hat R(t) &= \hat D_i(t) \quad \text{for each } i=2,\ldots,k. \label{eq:fj-fl-rate}
\end{align}
The fluid model equations for each individual server follow from the single-server analysis in \cref{sec:basic}.
\end{lemma}

\begin{proof}
By \cref{lem:join-bound-gen}, $\sum_i J_i(n)\le 2K'(k-1)$ for all~$n$. Dividing by $|x|\to\infty$ gives $\sum_i \hat J_i(t)=0$ in the limit. Since $\hat J_i(t)\ge 0$, each $\hat J_i(t)=0$. Combining with the scaled flow-balance equation $\hat J_i(t) = \hat J_i(0) + \hat D_i(t) - \hat R(t)$ and $\hat J_i(0)=0$ yields \cref{eq:fj-fl-rate}.
\end{proof}

\subsection{Fluid model stability}

\begin{proof}[Proof of \cref{thm:fork-join}]
We show fluid model stability, which implies positive recurrence by the argument in \cref{sec:stability}. Define the per-server Lyapunov function $f^s(t) = (m^s_p+m^s_d)\,Z^s_p(t) + m^s_d\,Z^s_d(t)$ for server~$q_s$, where $Z^s_p$ and $Z^s_d$ are the fluid-level prefill and decode request counts. Write $\delta_s = c_s - \lambda(m^s_p+m^s_d) > 0$ for the slack at server~$q_s$.

\textbf{Step 1: Server $q_1$ drains.}
Server~$q_1$ receives only external arrivals at rate~$\lambda$. When $f^1(t)>0$, work conservation gives $\dot f^1(t) = \lambda(m^1_p+m^1_d) - c_1 = -\delta_1 < 0$. By Lemma~8.5 of~\citep{dai2020processing}, $f^1(t)=0$ for $t\ge T_1 = f^1(0)/\delta_1$.

\textbf{Step 2: Fork servers $q_2,\ldots,q_k$ drain.}
For $t\ge T_1$, server~$q_1$ is in balance: departures occur at rate~$\lambda$, so each fork server receives arrivals at rate~$\lambda$. Before $T_1$, the fork servers may accumulate work from $q_1$'s initial drain, but $f^s(T_1)$ is bounded for each~$s$ (in the fluid model, $|Z(0)|\le 1$ and total departures from $q_1$ over $[0,T_1]$ are bounded by $c_1 T_1$). When $f^s(t)>0$, $\dot f^s(t) = \lambda(m^s_p+m^s_d) - c_s = -\delta_s < 0$. Therefore $f^s(t)=0$ for $t\ge T^s_2 = T_1 + f^s(T_1)/\delta_s$.

\textbf{Step 3: Join releases at rate $\lambda$.}
Let $T_2 = \max_{s\in\{2,\ldots,k\}} T^s_2$. For $t\ge T_2$, all fork servers are in balance with departure rate~$\lambda$. By \cref{eq:fj-fl-vanish}, the join buffers vanish in the fluid limit, so $\hat R(t) = \hat D_i(t)$ for each~$i$. Therefore the release rate $\dot R(t) = \lambda$ for $t\ge T_2$.

\textbf{Step 4: Server $q_{k+1}$ drains.}
For $t\ge T_2$, server~$q_{k+1}$ receives arrivals at rate~$\lambda$. When $f^{k+1}(t)>0$, $\dot f^{k+1}(t) = \lambda(m^{k+1}_p+m^{k+1}_d) - c_{k+1} = -\delta_{k+1} < 0$. Therefore $f^{k+1}(t)=0$ for $t\ge T_3 = T_2 + f^{k+1}(T_2)/\delta_{k+1}$.

All fluid queues drain by time~$T_3$, proving fluid stability. Positive recurrence follows by the argument in \cref{sec:stability}.
\end{proof}

%% file: sections/appendix_d_convex_hull_proof.tex
\section{Proof of \Cref{convex_hull} (Convex Hull)}
\label{sec:proof-convex-hull}

\begin{proof}[Proof of \cref{convex_hull}]
We prove the theorem in three steps: first establish a key monotonicity lemma, then verify the four vertices, and finally show that every feasible point lies in the convex hull.

\begin{lemma} \label{lem:monotonicity}
Let \( t_s = c + a \cdot \lceil s/b_0 \rceil \) with \( a > 0 \), \( c \ge 0 \), and \( b_0 > 0 \). Then for any integers \( x \ge 0 \), \( y \ge 0 \) with \( x + b_0 \le b_{\max} \),
\[
\frac{x + b_0}{t_{x + y + b_0}} \ge \frac{x}{t_{x + y}}.
\]
\end{lemma}

\begin{proof}
Let \( T = x + y \). Since \( \lceil (T + b_0)/b_0 \rceil = \lceil T/b_0 \rceil + 1 \), we have \( t_{T + b_0} = t_T + a \). Cross-multiplying, the desired inequality is equivalent to
\[
(x + b_0)\, t_T \ge x\, (t_T + a) \quad \Longleftrightarrow \quad b_0\, t_T \ge x\, a.
\]
Since \( \lceil T/b_0 \rceil \ge \lceil x/b_0 \rceil \ge x/b_0 \), we obtain
\[
b_0\, t_T = b_0\, c + a\, b_0 \lceil T/b_0 \rceil \ge b_0\, c + a\, x \ge a\, x,
\]
where the last inequality uses \( c \ge 0 \). This completes the proof.
\end{proof}

We now verify each vertex of the convex hull.

\paragraph{Vertex $A$.} The point \( A = \bigl( b_{\max}/t_{b_{\max}},\, 0 \bigr) \) corresponds to allocating all tokens to prefill. By \cref{lem:monotonicity} (with \( y = 0 \)), the ratio \( x/t_x \) satisfies \( (x + b_0)/t_{x+b_0} \ge x/t_x \) for all \( x \le b_{\max} - b_0 \). Applying this inductively from \( x = 0 \) in steps of \( b_0 \), and using \( b_0 \mid b_{\max} \), we conclude that \( b_{\max}/t_{b_{\max}} \) is the maximum achievable prefill rate. Hence \( A \) is a vertex.

\paragraph{Vertex $D$.} The point \( D = \bigl( 0,\, k_{\max}/t_{k_{\max}} \bigr) \) corresponds to decode-only batches using all \( k_{\max} \) slots. Since \( k_{\max} \le b_0 \), we have \( \lceil y/b_0 \rceil = 1 \) for all \( 1 \le y \le k_{\max} \), so \( t_y = c + a \) is constant in this range. Thus \( y/t_y = y/(c + a) \) is increasing in \( y \), and the decode rate is maximized at \( y = k_{\max} \). Hence \( D \) is a vertex.

\paragraph{Vertices $B$ and $C$.} These represent mixed batches with \( y = k_{\max} - 1 \) decode tokens at different prefill loads:
\begin{align*}
    B &= \Bigl(\frac{b_{\max} - k_{\max} + 1}{t_{b_{\max}}},\, \frac{k_{\max} - 1}{t_{b_{\max}}}\Bigr), &
    C &= \Bigl(\frac{b_0 - k_{\max} + 1}{t_{b_0}},\, \frac{k_{\max} - 1}{t_{b_0}}\Bigr).
\end{align*}

\paragraph{Every feasible point lies in the convex hull.}
Consider any feasible operating point \( (x/t_{x+y},\, y/t_{x+y}) \) with \( x + y \le b_{\max} \) and \( y \le k_{\max} - 1 \). Let \( m \) be the unique integer such that \( (m-1)b_0 < x + y \le m b_0 \le b_{\max} \). Then \( t_{x+y} = t_{m b_0} = c + am \), and since \( x \le m b_0 - y \),
\[
\frac{x}{t_{x+y}} \le \frac{m b_0 - y}{t_{m b_0}}, \qquad \frac{y}{t_{x+y}} \le \frac{y}{t_{m b_0}}.
\]
Thus the point is dominated (componentwise) by
\[
P_m = \Bigl( \frac{m b_0 - y}{t_{m b_0}},\, \frac{y}{t_{m b_0}} \Bigr).
\]
It suffices to show \( P_m \in \mathrm{conv}(O, A, B, C, D) \) for all valid \( m \) and \( 0 \le y \le k_{\max} - 1 \).

\textbf{Step 1.} For each fixed \( m \), define
\[
A_m = \Bigl( \frac{m b_0}{t_{m b_0}},\, 0 \Bigr), \qquad Q_m = \Bigl( \frac{m b_0 - k_{\max} + 1}{t_{m b_0}},\, \frac{k_{\max} - 1}{t_{m b_0}} \Bigr).
\]
Then \( P_m \) is a convex combination of \( A_m \) and \( Q_m \):
\[
P_m = \Bigl(1 - \frac{y}{k_{\max} - 1}\Bigr) A_m + \frac{y}{k_{\max} - 1}\, Q_m.
\]
By \cref{lem:monotonicity}, the sequence \( m b_0 / t_{m b_0} \) is nondecreasing in \( m \), so \( A_m \) lies on the segment \( \overline{OA} \) and hence in the convex hull. It remains to show \( Q_m \) lies in the convex hull.

\textbf{Step 2.} We show that \( Q_m \) lies on or below the line segment \( \overline{BC} \) for all \( 1 \le m \le b_{\max}/b_0 \). Note that \( Q_1 = C \) and \( Q_{b_{\max}/b_0} = B \), so the boundary cases hold with equality. For \( 1 < m < b_{\max}/b_0 \), let \( q = b_{\max}/b_0 \). The \( x \)-coordinates of \( B \), \( C \), and \( Q_m \) are
\[
x_B = \frac{q b_0 - k_{\max} + 1}{c + aq}, \quad x_C = \frac{b_0 - k_{\max} + 1}{c + a}, \quad x_{Q_m} = \frac{m b_0 - k_{\max} + 1}{c + am}.
\]
By \cref{lem:monotonicity}, \( x_{Q_m} \) is increasing in \( m \), so \( x_C \le x_{Q_m} \le x_B \). Thus there exists a unique \( \alpha \in [0,1] \) with \( x_{Q_m} = \alpha\, x_B + (1 - \alpha)\, x_C \).

The \( y \)-coordinate of \( Q_m \) is \( (k_{\max} - 1)/(c + am) \), while the corresponding point on \( \overline{BC} \) has \( y \)-coordinate
\[
y_{\mathrm{line}} = (k_{\max} - 1) \Bigl( \frac{\alpha}{c + aq} + \frac{1 - \alpha}{c + a} \Bigr).
\]
Define \( f(m) = 1/(c + am) \). Since \( a > 0 \), we have \( f''(m) = 2a^2/(c + am)^3 > 0 \), so \( f \) is strictly convex. By Jensen's inequality,
\[
f(m) < \alpha\, f(q) + (1 - \alpha)\, f(1)
\]
for \( 1 < m < q \), which gives
\[
\frac{k_{\max} - 1}{c + am} < (k_{\max} - 1) \Bigl( \frac{\alpha}{c + aq} + \frac{1 - \alpha}{c + a} \Bigr).
\]
Hence the \( y \)-coordinate of \( Q_m \) is strictly below \( y_{\mathrm{line}} \), confirming \( Q_m \) lies strictly below \( \overline{BC} \) and therefore inside \( \mathrm{conv}(O, A, B, C, D) \).

\textbf{Step 3: Decode-only batches.} Finally, consider decode-only batches with \( x = 0 \) and \( y = k_{\max} \). The operating point is \( (0, k_{\max}/t_{k_{\max}}) = D \), which is a vertex of the convex hull.

\textbf{Conclusion.} Every feasible point \( (x/t_{x+y},\, y/t_{x+y}) \) with \( x + y \le b_{\max} \) and \( y \le k_{\max} \) is dominated by a point in \( \mathrm{conv}(O, A, B, C, D) \). Conversely, each vertex is achievable by an appropriate batch composition, so the convex hull is tight.
\end{proof}